\documentclass[letterpaper]{article} % DO NOT CHANGE THIS
\usepackage{aaai23}  % DO NOT CHANGE THIS
\usepackage{times}  % DO NOT CHANGE THIS
\usepackage{helvet}  % DO NOT CHANGE THIS
\usepackage{courier}  % DO NOT CHANGE THIS
\usepackage[hyphens]{url}  % DO NOT CHANGE THIS
\usepackage{graphicx} % DO NOT CHANGE THIS
\usepackage{natbib}  % DO NOT CHANGE THIS AND DO NOT ADD ANY OPTIONS TO IT
\usepackage{caption} % DO NOT CHANGE THIS AND DO NOT ADD ANY OPTIONS TO IT
\usepackage{algorithm}
\usepackage{newfloat}
\usepackage{listings}
\DeclareCaptionStyle{ruled}{labelfont=normalfont,labelsep=colon,strut=off} % DO NOT CHANGE THIS
\floatstyle{ruled}
\newfloat{listing}{tb}{lst}{}
\floatname{listing}{Listing}
\nocopyright% -- Your paper will not be published if you use this command
\usepackage{amsmath,amssymb}
\usepackage[noend]{algpseudocode}
\usepackage{amsthm}

\newtheorem{lemma}{Lemma}

\title{A Metric Hybrid Planning Approach to Solving Pandemic Planning Problems with Simple SIR Models}
\author{
    Ari Gestetner, Buser Say
}
\affiliations{
    Monash University\\
    Melbourne, Australia\\
    ages0001@student.monash.edu, buser.say@monash.edu
}
\usepackage{bibentry}
\begin{document}

\maketitle

\begin{abstract}
A pandemic is the spread of a disease across large regions, and can have 
devastating costs to the society in terms of health, economic and social. 
As such, the study of effective pandemic mitigation strategies can yield 
significant positive impact on the society. A pandemic can be mathematically 
described using a compartmental model, such as the Susceptible–Infected–Removed 
(SIR) model. In this paper, we extend the solution equations of the SIR model 
to a state transition model with lockdowns. We formalize a metric 
hybrid planning problem based on this state transition model, and solve it 
using a metric hybrid planner. We improve the runtime 
effectiveness of the metric hybrid planner with the addition of valid 
inequalities, and demonstrate the success of our approach both theoretically 
and experimentally under various challenging settings.
\end{abstract}

\section{Introduction}

A pandemic is the spread of an infectious disease across large regions, 
and can incur devastating costs to the society in terms of health,
economic and social. A pandemic can be mathematically described as a 
compartmental model, such as the Susceptible–Infected–Removed 
(SIR) model.~\cite{Kermack1927, Bailey1975} The SIR model categorizes 
each individual in the population into three compartments, namely: 
either Susceptible, Infected or Removed, and uses a system of partial 
differential equations (PDEs) to describe the interaction between these 
compartments. As visualized in figure~\ref{fig:example}, the purpose of 
a pandemic planning problem is to find a solution that keeps the number 
of removed individuals under a certain threshold while also keeping the 
number of infected individuals under a certain threshold (i.e., without 
overwhelming the healthcare system) over some duration (e.g., until 
vaccinations are expected to become available), by deciding on the selection, 
timing and duration of effective mitigation decisions (e.g., lockdowns).

\begin{figure}[H]
    \centering
    \includegraphics[width=.48\linewidth]{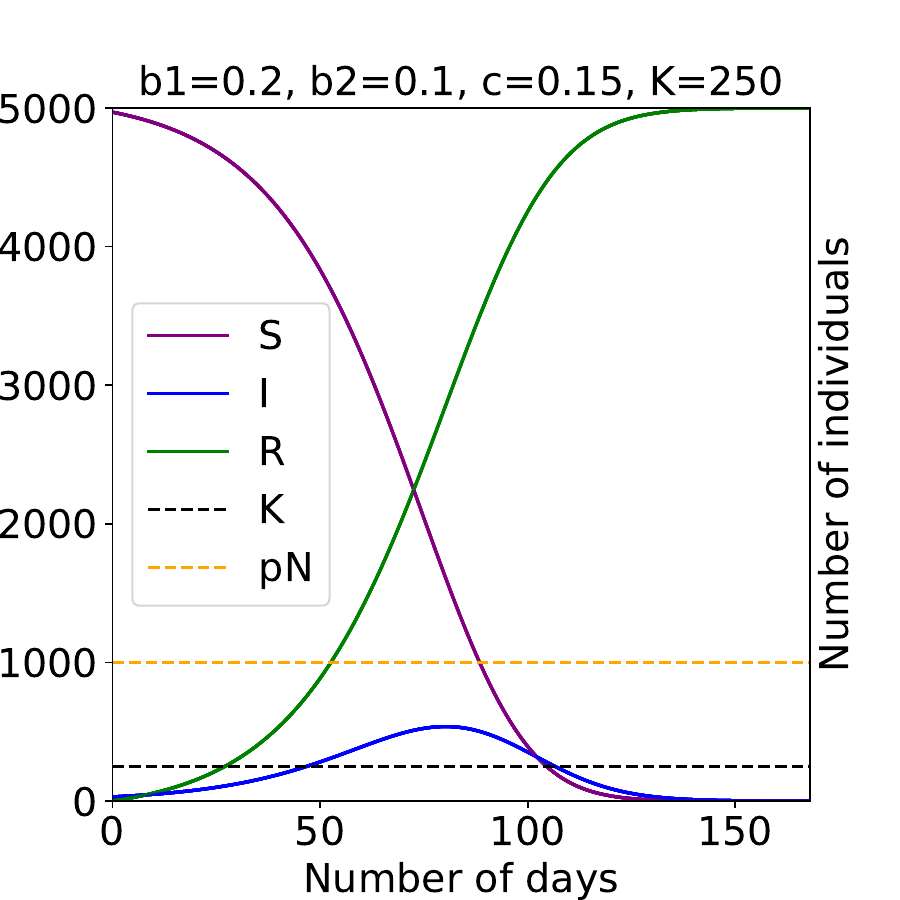}
    \label{fig:example_pandemic_0.2_0.1_0.15_250_noop}
    \centering
    \includegraphics[width=.48\linewidth]{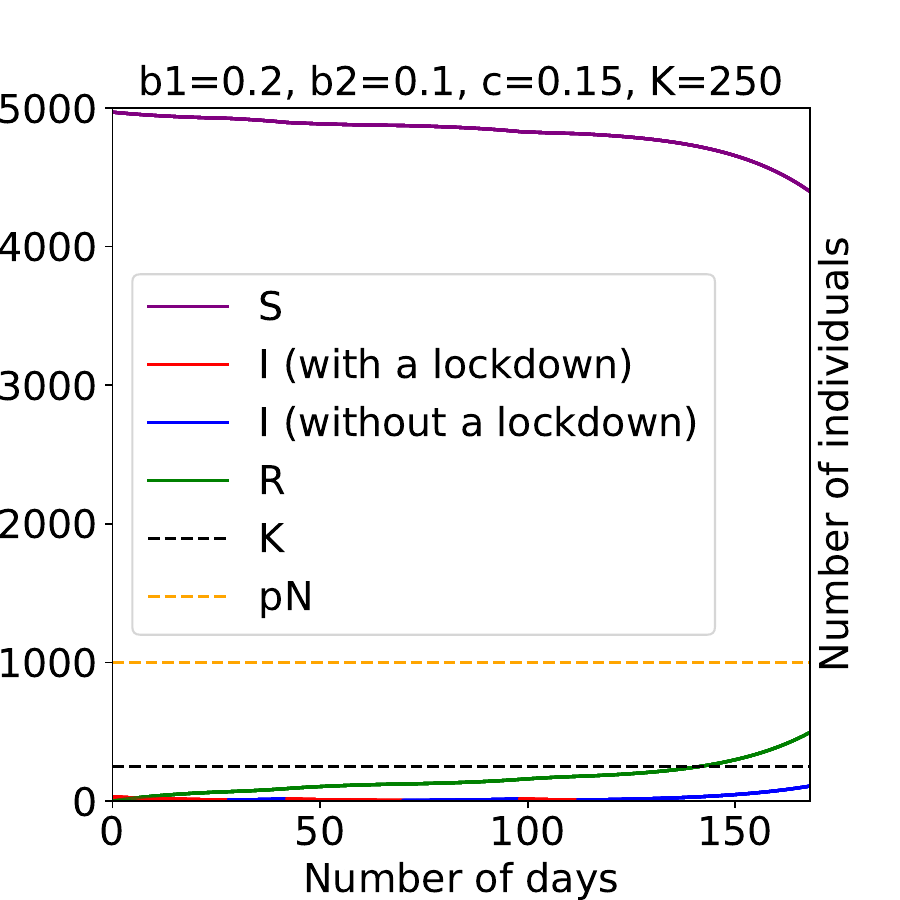}
    \label{fig:example_pandemic_0.2_0.1_0.15_250_plan}
\caption{Visualization of two pandemic curves based on the SIR model. The purple 
curves represent susceptible individuals, blue curves represent the infected 
individuals when there are no lockdowns and the green curves represent the 
removed individuals. The goal of pandemic planning problem is to keep the number 
of removed individuals under a certain threshold (i.e., orange dashed line) while 
also keeping the number of infected individuals under a certain threshold (i.e., 
black dashed line). On the left is the visualization of the natural spread of a 
disease without any lockdowns and on the right is the spread of the same disease 
under the plan that is produced by SCIPPlan using lockdowns (i.e., red lines).}
\label{fig:example}
\end{figure}

Automated planning formally reasons about the selection, timing and 
duration of actions to reach desired states of the world as best as 
possible.~\cite{Nau2004} Automated planners have significantly improved 
the ability of autonomous systems to solve challenging tasks such as, 
traffic control~\cite{McCluskey2017}, smart grid control~\cite{Thiebaux2013}, 
Heating, Ventilation and Air Conditioning (HVAC) control~\cite{Say2017a}, 
and Unmanned Aerial Vehicles (UAV) control~\cite{Ramirez2018}. While 
automated planning has been shown to successfully control many physical 
systems, the control of pandemics still presents non-trivial challenges 
to many of the existing automated planners mainly due to the complexity 
of the underlying system of equations that govern the evolution of the 
spread of the disease over time, which demands constrained and 
combinatorial sequential reasoning over PDEs. 

A common approach to sequential decision making with a PDE-based model
%a model that is based on a system of PDEs 
is to discretize the time and approximately update the state of the 
world with each discrete increment of 
time.~\cite{Morari1988,Li2008,Penna2009,Piotrowski2016,Scala2016} The main 
advantage of this approach is that it allows systems that can be modelled 
using PDEs with no known solution equations to be approximately controlled. 
The main disadvantage of this approach is that its accuracy and performance 
are highly sensitive to the arbitrary choice of the granularity of the time 
discretization. This can be particularly problematic for systems that 
exhibit exponential growth behaviour, such as the case in pandemics. An 
alternative approach here is to work directly with the solution equations 
of the PDEs. The main advantage of this approach is that it allows for the 
continuous time control of the underlying 
system.~\cite{Coles2012,Shin2005,Cashmore2016,Say2019b,Chen2021}
The main disadvantage of this approach is that not all PDEs have known 
solution equations. However, it should be noted that the advancements of 
machine learning techniques~\cite{Raissi2019,Karniadakis2021} allow for the 
accurate approximation of the solution equations of many complex systems of 
PDEs from (real and/or simulated) data, which makes this disadvantage less 
consequential in practise.

In this paper, we will solve the pandemic planning problem using an 
automated planner that is capable of synthesizing plans 
over continuous time. Specifically, we will begin by building a state 
transition model based on the solution equations of the SIR 
model~\cite{Bailey1975,Gleissner1988,Bohner2019} with infection rate 
that is a function of the lockdown decision. Given the state transition 
model, we will formalize the pandemic planning problem as a metric hybrid 
planning problem~\cite{Say2019b}. We will introduce domain-dependent 
valid inequalities in order to improve the computational effectiveness 
of the automated planner. We will show theoretical results on the 
finiteness and the correctness of the automated planner for solving 
the pandemic planning problem, and experimentally validate the 
success of our approach under various challenging settings.
 
\section{Background}

We begin by presenting the definition of the metric hybrid planning problem 
that will be used to formalize the pandemic planning problem, a metric hybrid 
planner for solving the pandemic planning problem, and the SIR model that will 
form the basis of the pandemic planning problem.

\subsection{The Metric Hybrid Planning Problem}

\newcommand{\nn}{n} % |S|
\newcommand{\mm}{m} % |A|

A \emph{metric hybrid planning problem}~\cite{Say2018b,Say2019b,Say2023a} 
is defined as a tuple ${\Pi} = \langle S,A,\Delta,C^{I},C^{T},T,V,G,R,H \rangle$ where 
\begin{itemize}
    \item $S=\{s_1, \dots ,s_{\nn}\}$ is the set of state variables 
with bounded domains $D_{s_1}, \dots, D_{s_{\nn}}$ for positive integer 
$\nn \in \mathbb{Z}^{+}$,
    \item $A=\{a_1, \dots ,a_{\mm}\}$ is the set of action variables 
with bounded domains $D_{a_1}, \dots, D_{a_{\mm}}$ for positive integer 
$\mm \in \mathbb{Z}^{+}$,
    \item $\Delta \in [\epsilon,M]$ is the duration of a step for positive real 
numbers $\epsilon \in \mathbb{R}^{+}$ and $M \in \mathbb{R}^{+}$ such that $\epsilon \leq M$,
    \item $C^{I}: \prod_{i=1}^{\nn} D_{s_i} \times \prod_{i=1}^{\mm} D_{a_i}
\times [\epsilon,M] \rightarrow \mathbb{R}$ is the instantaneous 
constraint function that is used to define the constraint 
$C^{I}(s^{t}_{1}, \dots, s^{t}_{\nn}, a^{t}_{1}, \dots, a^{t}_{\mm}, \Delta^{t}) 
\leq 0$ for all steps $t \in \{1,\dots,H\}$,
    \item $C^{T}: \prod_{i=1}^{\nn} D_{s_i} \times \prod_{i=1}^{\mm} D_{a_i}
\times [\epsilon,M] \rightarrow \mathbb{R}$ is the temporal 
constraint function that is used to define the constraint 
$C^{T}(s^{t}_{1}, \dots, s^{t}_{\nn}, a^{t}_{1}, \dots, a^{t}_{\mm},\Delta^{t}) 
\leq 0$ for all steps $t \in \{1,\dots,H\}$,
    \item $T: \prod_{i=1}^{\nn} D_{s_i} \times \prod_{i=1}^{\mm} D_{a_i}
\times [\epsilon,M] \rightarrow \prod_{i=1}^{\nn} D_{s_i}$ 
is the state transition function,
    \item $V$ is a tuple of constants $\langle V_1,\dots, V_{\nn } \rangle \in 
    \prod_{i=1}^{\nn} D_{s_i}
    $ denoting the initial values of state variables,
    \item $G: \prod_{i=1}^{\nn} D_{s_i} \rightarrow \mathbb{R}$ 
is the goal function that is used to define 
the constraint $G(s^{H+1}_{1}, \dots, s^{H+1}_{\nn}) \leq 0$,
    \item $R: \prod_{i=1}^{\nn} D_{s_i} \times \prod_{i=1}^{\mm} D_{a_i}
\times [\epsilon,M] \rightarrow \mathbb{R}$ is the reward 
function, and
    \item $H\in \mathbb{Z}^{+}$ is the horizon.
\end{itemize}

A \emph{solution} to $\Pi$ is defined as a tuple of values $\langle \bar{a}^{t}_{1}, 
\dots, \bar{a}^{t}_{\mm} \rangle \in \prod_{i=1}^{\mm} D_{a_i}$ for all 
action variables $A$ and a value $\bar{\Delta}^{t} \in [\epsilon,M]$ for 
all steps $t\in \{1,\dots, H\}$ and a tuple of values $\langle \bar{s}^{t}_{1}, 
\dots, \bar{s}^{t}_{\nn} \rangle \in \prod_{i=1}^{\nn} D_{s_i}$ for all state 
variables $S$ and steps $t\in \{1,\dots, H+1\}$ if and only if the following 
conditions hold:
\begin{enumerate}
    \item $\bar{s}^{1}_{i} = V_i$ for all $i \in \{1, \dots, \nn\}$,
    \item $T(\bar{s}^{t}_{1}, \dots, \bar{s}^{t}_{\nn }, \bar{a}^{t}_{1}, 
\dots, \bar{a}^{t}_{\mm}, \bar{\Delta}^{t}) = \langle \bar{s}^{t+1}_{1}, 
\dots, \bar{s}^{t+1}_{\nn } \rangle$ for steps $t\in \{1,\dots, H\}$,
    \item $C^{I}(\bar{s}^{t}_{1}, \dots, \bar{s}^{t}_{\nn }, \bar{a}^{t}_{1}, 
\dots, \bar{a}^{t}_{\mm}, \bar{\Delta}^{t}) \leq 0$ for steps $t\in \{1,\dots, H\}$,
    \item $C^{T}(\bar{s}^{t}_{1}, \dots, \bar{s}^{t}_{\nn }, \bar{a}^{t}_{1}, 
\dots, \bar{a}^{t}_{\mm}, x^{t}) \leq 0$ for steps $t\in \{1,\dots, H\}$ 
and for all values of $x^{t} \in [0,\bar{\Delta}^{t}]$, and
    \item $G(\bar{s}^{H+1}_{1}, \dots, \bar{s}^{H+1}_{\nn }) 
\leq 0$.
\end{enumerate}

Similarly, an \emph{optimal solution} to $\Pi$ is a solution that also maximizes 
the reward function $R$ over the planning horizon $H$ such that:
$$\max_{\substack{a^{1}_{1}, \dots, a^{H}_{\mm} \\ \Delta^{1}, \dots, \Delta^{H}}} 
\sum_{t=1}^{H}R(s^{t}_{1}, \dots, s^{t}_{\nn }, a^{t}_{1}, \dots, 
a^{t}_{\mm}, \Delta^{t})$$
Next, we will present a methodology for solving $\Pi$.

\subsection{Solving the Metric Hybrid Planning Problem}
\label{sec:solving_pi}

SCIPPlan~\cite{Say2018b,Say2019b,Say2023a} is a metric hybrid planner that 
performs on mathematical optimization. SCIPPlan compiles $\Pi$ into the 
mathematical optimization model that is provided below.

\begin{align}
&\max_{\substack{a^{1}_{1}, \dots, a^{H}_{\mm} \\ \Delta^{1}, \dots, \Delta^{H}}} 
\sum_{t=1}^{H}R(s^{t}_{1}, \dots, s^{t}_{\nn }, 
a^{t}_{1}, \dots, a^{t}_{\mm}, \Delta^{t})\label{scip0}\\
&s^{1}_{i} = V_i \quad \forall_{i\in \{1,\dots, \nn\}}\label{scip1}\\
&T_i(s^{t}_{1}, \dots, \Delta^{t}) = s^{t+1}_{i} \quad 
\forall_{i \in \{1, \dots, \nn\}, t\in \{1,\dots, H\}}\label{scip2}\\
&C^{I}(s^{t}_{1}, \dots, \Delta^{t}) \leq 0 \quad 
\forall_{t\in \{1, \dots, H\}}\label{scip3}\\
&C^{T}(s^{t}_{1}, \dots, c^{t} \Delta^{t}) \leq 0 \quad 
\forall_{c^{t}\in [0,1], t\in \{1, \dots, H\}}\label{scip4}\\
&G(s^{H+1}_{1}, \dots, s^{H+1}_{\nn}) \leq 0 \label{scip5}\\
&s^{t}_{i} \in D_{s_i} \quad \forall_{i\in \{1, \dots, \nn\}, t\in \{1, \dots, H+1\}} 
\label{scip6}\\ 
&a^{t}_{i} \in D_{a_i} \quad \forall_{i\in \{1, \dots, \mm\}, t\in \{1, \dots, H\}} 
\label{scip7}\\
&\epsilon \leq \Delta^{t} \leq M \quad \forall_{t\in \{1, \dots, H\}}\label{scip8}
\end{align}

The brief description of the mathematical optimization model is as follows. 
The objective function (\ref{scip0}) maximizes the total reward accumulated 
over the horizon. Constraint (\ref{scip1}) sets the initial values of all 
state variables. Constraint (\ref{scip2}) sets the values of state variables 
in the next step given the values of state and action variables in the current 
step. Constraints (\ref{scip3}-\ref{scip5}) enforce the instantanous, temporal 
and goal constraints. Constraints (\ref{scip6}-\ref{scip8}) set the domains 
of state, action and duration variables, respectively. Note that in the mathematical 
optimization model presented above, constraint (\ref{scip4}) is an interval 
constraint that holds for all values of the interval $c^{t}\in [0,1]$. SCIPPlan 
uses a constraint generation framework that iteratively identifies unique 
violated values of $c^{t}\in [0,1]$ and generates constraint (\ref{scip4}) 
with the indentified violated values of $c^{t}$. It has been shown that SCIPPlan 
terminates either with a solution to $\Pi$ with at most some constant $\gamma$ 
violation of constraint (\ref{scip4}) or proves the infeasibility of $\Pi$ in 
finite number of constraint generation iterations.\footnote{The proof assumes 
the temporal constraint function $C^{T}$ to be Lipschitz continuous. Other 
control functions can be used to derive domain-specific bounds based on the 
analysis of $C^{T}$, as we will show in the Theoretical Results section.}
Next, we will describe the SIR model that will be used to construct the state 
transition function $T$ of our planning problem $\Pi$.

\subsection{The SIR Model with Constant Infection and Recovery Rates}
\label{sec:sir}

The SIR model~\cite{Kermack1927} is a compartmental model that groups a 
population into three compartments, namely: Susceptible, Infected and 
Removed. The SIR model mathematically describes the interaction between 
these groups using a system of partial differential equations. In this 
paper, we will build a state transition function $T$ based on the solution 
equations of the SIR model with constant infection $b\in [0,1]$ and removal 
$c\in [0,1]$ rates.~\cite{Bailey1975,Gleissner1988,Bohner2019} The system of 
partial differential equations:
\begin{align}
&x' = -\frac{bxy}{x+y}\label{sus}\\
&y' = \frac{bxy}{x+y} - cy\label{inf}\\
&z' = cy\label{rem}
\end{align}
has the following solution equations for $b\neq c$:
\begin{align}
&x^{t+1} = x^{t}(1+\frac{y^{t}}{x^{t}})^{\frac{b}{b-c}} 
(1 + \frac{y^{t}}{x^{t}} e^{(b-c)\Delta^t})^{-\frac{b}{b-c}}\label{sus_sol}\\
&y^{t+1} = y^{t}(1+\frac{y^{t}}{x^{t}})^{\frac{b}{b-c}} 
(1 + \frac{y^{t}}{x^{t}} e^{(b-c)\Delta^t})^{-\frac{b}{b-c}} 
e^{(b-c)\Delta^t}\label{inf_sol}\\
&z^{t+1} = N - (x^{t} + y^{t})^{\frac{b}{b-c}} 
(x^{t} + y^{t}e^{(b-c)\Delta^t})^{-\frac{c}{b-c}}\label{rem_sol}
\end{align}
where $x^t$, $y^t$ and $z^t$ represents the number of susceptible, 
infected and removed individuals in a population of size 
$N\in \mathbb{Z}^{+}$. In this paper, equations (\ref{sus_sol}-\ref{rem_sol}) 
will form the basis of the state transition function $T$ that will be used 
to formalize the pandemic planning problem as a metric hybrid planning problem 
$\Pi$.
%Before we present the contributions of our paper, we will next detail the underlying assumptions that form the basis of our proposed approach. 
%\subsection{Model Assumptions}
%\label{sec:assumptions}
%Our proposed approach will make the following assumptions. First, we will assume that the size of the population $N$, the infection rate $b\in [0,1]$ and the removal rate $c\in [0,1]$ are constant. This assumption is fundamental to the solution equations of the SIR model presented in section~\ref{sec:sir}. Second, we will assume that the total cost of the pandemic is primarily made out of three components, namely: (i) the cost of overwhelming the healthcare system (e.g., mortality rate can be significantly higher for untreated infected individuals), (ii) the cost associated with the total duration of the pandemic (i.e., macro level pandemic policies that continue over the duration of the pandemic can incur health, economic and social costs) and (iii) the total duration of lockdowns used to control the pandemic (i.e., health, economic and social costs for individuals such as not receiving critical medical treatments, not being able to work and decline in mental health etc.). Later in section~\ref{sec:discussion}, we will discuss how to model alternative assumptions.
Next, we will present the contributions of our paper, namely: 
(i) the formalization of the pandemic planning problem as a metric hybrid 
planning problem and (ii) the valid inequalities to help improve the 
effectiveness of the metric hybrid planner.

\section{Pandemic Planning Using a Metric Hybrid Planner}

In this section, we will demonstrate how the pandemic planning problem 
can be effectively solved using SCIPPlan.

\subsection{The State Transition Model with Lockdowns}

We will extend the solution equations of the previously 
described SIR model to build a state transition function $T$ with lockdowns. 
The inputs of function $T$, namely the elements of the tuple 
$\langle s_{1}, s_{2}, s_{3}, a_{1}, \Delta \rangle$, are defined as follows.
\begin{itemize}
    \item $s_1, s_2 ,s_3 \in [0,N]$ are the bounded state variables 
representing the number of susceptible $x$, infected $y$ and removed 
$z$ individuals in a population of size $N$,
    \item $\Delta$ is the fixed duration of a step with the constant 
value $\delta \in [\epsilon, M]$, and
    \item $a_1 \in \{0,1\}$ is the binary action variable representing 
whether the population is under lockdown (i.e., $a_1 = 1$) or not 
(i.e., $a_1 = 0$) over the duration $\Delta$.
\end{itemize}

Similarly, the output of function $T$ is defined as the tuple of state 
variables $\langle s_{1}, s_{2}, s_{3} \rangle$. Next, we will modify 
equations (\ref{sus_sol}-\ref{rem_sol}) with the addition of the binary 
action variable $a^{t}_{1}$ which will determine the value of the 
infection rate $b(a_1)$ over the duration $\Delta$. The value of the 
infection rate $b(a_1)$ is defined as:
\begin{align}\label{b_new}
    b(a_1) = 
    \begin{cases}
    b_1,& \text{if } a_1=1\\
    b_2,& \text{otherwise.}
\end{cases}
\end{align}
where constants $b_1, b_2\in [0,1]$ represent the infection rates with 
and without a lockdown, respectively, such that $b_1 > b_2$. Given the 
new definition of the infection rate $b(a_1)$, the state transition function 
$T$ can be written by symbolically substituting every occurrence of $x^{t+1}$, $y^{t+1}$, $z^{t+1}$, $x^{t}$, $y^{t}$, $z^{t}$ and $b$ in equations (\ref{sus_sol}-\ref{rem_sol}) with $s_{1}^{t+1}$, $s_{2}^{t+1}$, $s_{3}^{t+1}$, $s_{1}^{t}$, $s_{2}^{t}$, $s_{3}^{t}$ and $b(a_1)$, respectively.
%\begin{align}
%&s_{1}^{t+1} = s_{1}^{t}(1+\frac{s_{2}^{t}}{s_{1}^{t}})^{\frac{b(a_1)}{b(a_1)-c}} 
%(1 + \frac{s_{2}^{t}}{s_{1}^{t}} e^{(b(a_1)-c)\Delta^t})^{-\frac{b(a_1)}{b(a_1)-c}}\label{sus_T}\\
%&s_{2}^{t+1} = s_{2}^{t}(1+\frac{s_{2}^{t}}{s_{1}^{t}})^{\frac{b(a_1)}{b(a_1)-c}} 
%(1 + \frac{s_{2}^{t}}{s_{1}^{t}} e^{(b(a_1)-c)\Delta^t})^{-\frac{b(a_1)}{b(a_1)-c}} 
%e^{(b(a_1)-c)\Delta^t}\label{inf_T}\\
%&s_{3}^{t+1} = N - (s_{1}^{t} + s_{2}^{t})^{\frac{b(a_1)}{b(a_1)-c}} 
%(s_{1}^{t} + s_{2}^{t}e^{(b(a_1)-c)\Delta^t})^{-\frac{c}{b(a_1)-c}}\label{rem_T}
%\end{align}
Next, we will formalize the pandemic planning problem as a metric hybrid 
planning problem $\Pi$.

\subsection{The Pandemic Planning Problem}

The remaining components of the metric hybrid planning problem $\Pi$ 
can be defined to describe the pandemic planning problem as follows.

\begin{itemize}
    \item $C^{T}(s^{t}_{1}, s^{t}_{2}, s^{t}_{3}, a^{t}_{1}, \Delta^{t}) \leq 0$ is 
a constraint on the maximum number of infected individuals allowed at any given time. 
The temporal constraint enforces that the right-hand side of the state transition 
function for the state variable $s_2^{t+1}$ never exceeds the constant threshold 
$K\in \mathbb{Z}^{+}$. Similarly, the domain of the state variable 
$s_2^t$ that represents the number of infected individuals is also bounded 
by $K$ such that $s_2^t \in [0,K]$. 
    \item $V_1$, $V_2$ and $V_3$ are the initial values of the number of susceptible, 
infected and removed individuals in the population, and are defined as $V_1 = N - I$, 
$V_2 = I$ and $V_3 = 0$ where the constant $I\in \mathbb{Z}^{+}$ denotes the initial 
number of infected individuals.
    %\item $G(s^{H+1}_{1}, s^{H+1}_{2}, s^{H+1}_{3}) \leq 0$ is a constraint on the minimum fraction $p\in [0,1]$ of the population that must be removed, and is in the form of: $s^{H+1}_{3} \geq p N$.
    \item $G(s^{H+1}_{1}, s^{H+1}_{2}, s^{H+1}_{3}) \leq 0$ is a constraint on the 
maximum fraction $p\in [0,1]$ of the population that can be removed, and is in the 
form of: $s^{H+1}_{3} \leq p N$.
    \item $R(s^{t}_{1}, s^{t}_{2}, s^{t}_{3}, a^{t}_{1}, \Delta^{t})$ is the reward 
function that incurs a penalty for a lockdown, and is in the form of: $-a^{t}_1$.
\end{itemize}

Given the components of $\Pi$ are defined, SCIPPlan can compile $\Pi$ into a 
mathematical optimization model and solve it via constraint generation 
(i.e., previously described in the Solving the Metric Hybrid Planning Problem section). 
%with increasing values of horizon in order to first minimize the total duration of the pandemic and second minimize the total number of lockdowns used over the horizon. 
However, 
the computational effectiveness of the underlying spatial branch-and-bound solver 
relies on its ability to effectively derive valid upper and lower bounds on the 
(nonlinear) expressions that make up the mathematical optimization model (i.e., 
the objective function (\ref{scip0}) and the constraints (\ref{scip1}-\ref{scip5})). 
Therefore, we will next derive valid inequalities in order to improve the 
computational effectiveness of SCIPPlan.

\subsection{Valid Inequalities}
\label{sec:valid_ineq}

In this section, we will introduce domain-dependent valid inequalities in order to 
improve the computational effectiveness of SCIPPlan. Namely, we will introduce two 
types of valid inequalities that are based on (i) the monotonicity of state variables 
$s_1$ and $s_3$, and (ii) the assumption that the size of the population is constant.
\begin{enumerate}
    \item Monotonic state variables: The analysis of equations (\ref{sus}) and (\ref{rem}) reveals 
that $s_1$ is non-increasing and $s_3$ is non-decreasing.
    \item Constant population: The SIR model~\cite{Bailey1975} assumes that the size 
of the population remains constant.
\end{enumerate}

The monotonicity of state variables and the constant population size assumption 
are used to define the instantaneous constraint $C^{I}(s^{t}_{1}, s^{t}_{2}, 
s^{t}_{3}, a^{t}_{1}, \Delta^{t}) \leq 0$ and redefine the goal constraint 
$G(s^{H+1}_{1}, s^{H+1}_{2}, s^{H+1}_{3}) \leq 0$.

\subsection{The Mathematical Optimization Model}

In this section, we present the mathematical optimization model 
that is compiled by SCIPPlan to represent and solve the pandemic 
planning problem.

\begin{align}
&\max_{a^{1}_{1}, \dots, a^{H}_{1}} 
\sum_{t=1}^{H}-a^{t}_1\label{scip0_pand}\\
&s^{1}_{1} = N - I \label{scip1_pand}\\
&s^{1}_{2} = I \label{scip2_pand}\\
&s^{1}_{3} = 0 \label{scip3_pand}\\ 
&s^{t+1}_{1} = s^{t}_{1}(1+\frac{s^{t}_{2}}{s^{t}_{1}})^{e_1} 
(1 + \frac{s^{t}_{2}}{s^{t}_{1}} e^{e_2})^{-e_1} 
\quad \forall_{t\in \{1,\dots, H\}}\label{scip4_pand}\\
&s^{t+1}_{2} = s^{t}_{2}(1+\frac{s^{t}_{2}}{s^{t}_{1}})^{e_1} 
(1 + \frac{s^{t}_{2}}{s^{t}_{1}} e^{e_2})^{-e_1} e^{e_2} 
\quad \forall_{t\in \{1,\dots, H\}}\label{scip5_pand}\\
&s^{t+1}_{3} = N - (s^{t}_{1} + s^{t}_{2})^{e_1} 
(s^{t}_{1} + s^{t}_{2}e^{e_2})^{-e_3} 
\quad \forall_{t\in \{1,\dots, H\}}\label{scip6_pand}\\
&s_{1}^{t} + s_{2}^{t} + s_{3}^{t} = N \quad 
\forall_{t\in \{1, \dots, H\}}\label{scip7_pand}\\
&(1+\frac{s_{2}^{t}}{s_{1}^{t}})^{e_1} 
(1 + \frac{s_{2}^{t}}{s_{1}^{t}} e^{(b(a_1)-c) \Delta^t})^{-e_1} 
\leq 1 \quad \forall_{t\in \{1, \dots, H\}} \label{scip8_pand}\\ 
&(s_{1}^{t} + s_{2}^{t})^{e_1} (s_{1}^{t} + s_{2}^{t} 
e^{e_2})^{-e_3} + s_{3}^{t} \leq N \quad 
\forall_{t\in \{1, \dots, H\}}\label{scip9_pand}\\
&s^{t}_{2} (1+\frac{s^{t}_{2}}{s^{t}_{1}})^{e_1} 
(1 + \frac{s^{t}_{2}}{s^{t}_{1}} e^{e_2})^{-e_1} e^{e_2} \leq K \quad 
\forall_{c^{t}\in [0,1], t\in \{1, \dots, H\}}\label{scip10_pand}\\
&s_{1}^{H+1} + s_{2}^{H+1} + s_{3}^{H+1} = N \label{scip11_pand}\\
&s^{H+1}_{3} \leq p N \label{scip12_pand}\\
&s^{t}_{1} \in [0,N] \quad \forall_{t\in \{1, \dots, H+1\}} 
\label{scip13_pand}\\ 
&s^{t}_{2} \in [0,K] \quad \forall_{t\in \{1, \dots, H+1\}} 
\label{scip14_pand}\\ 
&s^{t}_{3} \in [0,N] \quad \forall_{t\in \{1, \dots, H+1\}} 
\label{scip15_pand}\\
&a^{t}_{1} \in \{0,1\} \quad \forall_{t\in \{1, \dots, H\}} 
\label{scip16_pand}\\
&\Delta^{t} = \delta \quad \forall_{t\in \{1, \dots, H\}}\label{scip17_pand}
\end{align}

where $e_1$, $e_2$ and $e_3$ denote the expressions 
$\frac{b(a_1)}{b(a_1)-c}$, $(b(a_1)-c)\Delta^t$ and 
$\frac{c}{b(a_1)-c}$, respectively. The brief description 
of the resulting mathematical optimization model is as 
follows. The objective function (\ref{scip0_pand}) 
minimizes the total number of lockdowns used over 
the horizon. Constraints 
(\ref{scip1_pand}-\ref{scip3_pand}) set the initial 
values of all state variables. Constraints 
(\ref{scip4_pand}-\ref{scip6_pand}) represent the 
state transition function $T$. Constraints 
(\ref{scip7_pand}-\ref{scip9_pand}) are the valid 
inequalities that are introduced in the Valid Inequalities
section. Constraint 
(\ref{scip10_pand}) is the temporal constraint. 
Constraints (\ref{scip11_pand}-\ref{scip12_pand}) 
are the goal state constraints. Finally, constraints 
(\ref{scip13_pand}-\ref{scip17_pand}) specify the 
domains of state and action variables, and the fixed 
duration of each step.

\section{Theoretical Results}
\label{sec:proof}

In this section, we present our theoretical results on the finiteness 
and correctness of SCIPPlan for solving the pandemic planning problem 
with an exponential temporal constraint function $C^{T}$. 

\begin{lemma}[Finiteness and correctness of SCIPPlan]
\label{lemma:scipplan_finiteness_exp}
SCIPPlan finds a solution to the pandemic planning problem within 
some constant $\gamma > 0$ constraint violation tolerance or proves 
its infeasibility in finite number of constraint generation iterations.
\end{lemma}

\begin{proof}
We begin our proof with the analysis of equation (\ref{inf_sol}) that 
governs the evolution of the number of infected individuals $y$. The 
joint analysis of equations (\ref{sus}), (\ref{sus_sol}) and 
(\ref{inf_sol}) reveals that the expression:
\begin{align}
&(1+\frac{y^{t}}{x^{t}})^{\frac{b}{b-c}} (1 + \frac{y^{t}}{x^{t}} 
e^{(b-c)\Delta^t})^{-\frac{b}{b-c}}\label{theor_1}
\end{align}
is bounded between 0 and 1. This means that the growth of the number 
of infected individuals $y$ must be bounded by the expression:
\begin{align}
&K e^{(b-c)\Delta^t}\label{theor_2}
\end{align}
which concludes our proof since expression (\ref{theor_2}) can 
be used as the control function in the analysis of the temporal 
constraint function $C^{T}$ instead of the Lipschitz function 
that was used in the original proof.~\cite{Say2023a}
\end{proof}

\section{Related Problems and Extensions}

In this section, we introduce a modification of the pandemic planning 
problem that frequently appears in many other areas of social physics, 
and also show how to extend the pandemic planning problem to use variable 
step duration. We begin with the presentation of a modification of the 
pandemic planning problem that frequently appears in other facets of social 
physics, such as migration, networks and communities.~\cite{Jusup2022} 
In this new related planning problem, the total cost is primarily made out 
of two components, namely: (i) the cost associated with the total duration 
(i.e., makespan) and (ii) the total action duration.
In order to solve this new problem, we modify the goal constraint 
(\ref{scip12_pand}) to:
\begin{align}
&s^{H+1}_{3} \geq q N\label{scip_alternative}
\end{align}
where $q\in [0,1]$ is the minimum fraction $p\in [0,1]$ of the population 
that must reach a certain threshold, and run SCIPPlan with increasing values 
of horizon $H$. For simplicity, we will refer to the previous version of the 
planning problem as Problem 1, and refer to the new version as Problem 2.

We proceed with the extension of the pandemic planning problem that uses 
variable step duration. In order to achieve variable step duration, we modify 
the reward function to:
\begin{align}
&\max_{\substack{a^{1}_{1}, \dots, a^{H}_{1} \\ \Delta^{1}, \dots, \Delta^{H}}}
\sum_{t=1}^{H}-\Delta^{t} a^{t}_1\label{scip18_pand}
\end{align}
and constraint (\ref{scip17_pand}) to:
\begin{align}
&\delta_{LB} \leq \Delta^{t} \leq \delta_{UB} \quad \forall_{t\in \{1, \dots, H\}}\label{scip19_pand}
\end{align}
where $\delta_{LB}$ and $\delta_{UB}$ denote the lower and upper 
bounds on the duration of each step, respectively. Finally, in order 
to cover a fixed total duration, we add the following constraint:
\begin{align}
&\sum_{t=1}^{H}\Delta^{t} = F\label{scip20_pand}
\end{align}
where $F$ denotes the fixed total duration.

\section{Experimental Results}
\label{sec:experiments}

In this section, we present the results of our detailed computational 
experiments for testing the effectiveness of using SCIPPlan to perform 
planning under various problem settings.

\subsection{Design of Experiments}

We have conducted two sets of computational experiments. First, we tested 
SCIPPlan under both problem settings (i.e., Problem 1 and 2) over 96 
instances where each instance corresponds to a unique combination of the 
parameters provided in Table~\ref{tab:params}. Second, we tested the 
effect of using variable step duration on the solution quality and the 
runtime performance of SCIPPlan.

\begin{table}
\scriptsize
  \centering
    \begin{tabular}{| c | c | c |}
    \hline
    Parameter & Values & Brief Description \\ \hline
    $b_1$ & 0.2, 0.25 & Infection rate without a lockdown. \\ \hline
    $b_2$ & 0.1, 0.15 & Infection rate with a lockdown. \\ \hline
    $c$ & 0.15, 0.2 & Removal rate. \\ \hline
    $N$ & 5000 & Total population size. \\ \hline
    $K$ & 200, 250 & Maximum infected allowed. \\ \hline
    $I$ & 30, 40, 50, 60 & Initial infected. \\ \hline
    $p$ & 0.2 & Maximum fraction removed. \\ \hline
    $q$ & 0.8 & Minimum fraction removed. \\ \hline
    $\delta$ & 14, 21, 28 & Fixed duration of a step. \\ \hline
    $[\delta_{LB},\delta_{UB}]$ & [7,28] & Variable duration of a step. \\ \hline
    $F$ & 168 & Fixed total duration. \\ \hline
  \end{tabular}
  \caption{Summary of the pandemic parameters that are used to create the instances.} %detailed in the Appendix.
  \label{tab:params}
\end{table}

\subsection{Experimental Setup}

All experiments were run on the Apple M1 Chip with 16GB memory, using 
a single thread with one hour total time limit per instance. SCIPPlan 
uses SCIP~\cite{Vigerske2018} as its spatial branch-and-bound solver.

\subsection{Implementation Details}

In all of our experiments, we linearized equation (\ref{b_new}) using 
a system of linear constraints. We added the new linearization constraints 
as a part of the instantaneous constraint $C^{I}(s^{t}_{1}, s^{t}_{2}, 
s^{t}_{3}, a^{t}_{1}, \Delta^{t}) \leq 0$. We encoded each exponential 
term of the form $a^b$ as $e^{b log(a)}$. We verified the temporal 
constraint function $C^{T}$ with increments of $0.1$.

In order to solve Problem 1 with fixed step duration, we ran SCIPPlan 
with the value of $H$ such that $H\Delta$ is equal to 24 weeks. In order 
to solve Problem 2 with fixed step duration, we ran SCIPPlan with increasing 
values of horizon $H$ until either a solution is found or $H\Delta$ exceeded 
52 weeks. If SCIPPlan returned a solution with $H\Delta$ that is less than 52 
weeks, we verified that the temporal constraint $C^{T}(s^{t}_{1}, s^{t}_{2}, 
s^{t}_{3}, a^{t}_{1}, \Delta^{t}) \leq 0$ holds for the remaining 
duration of the year without using additional actions (i.e., $a_1^t = 0$), 
otherwise we ran SCIPPlan with the incremented value of $H$. Finally, 
in order to solve Problem 1 with variable step duration, we have used 
$\delta_{LB} = 7$ and $\delta_{UB} = 28$ as the lower and upper bounds 
of the duration of each step with horizon $H=8$ and fixed total duration $F=168$  (i.e., 24 weeks). 

%c^{blog_c a}

%\subsection{Effect of Pandemic Parameters on Runtime Performance}

%\subsection{Effect of Pandemic Parameters on Coverage and Solution Quality}

\subsection{Effect of Valid Inequalities}

We begin our analysis with the evaluation of the effect of valid 
inequalities that are introduced in the Valid Inequalities section 
(i.e., constraints (\ref{scip7_pand}-\ref{scip9_pand})) on the 
overall runtime performance of SCIPPlan. 
Figure~\ref{fig:pairwise_timing} visualizes the logarithmic runtime 
comparison between using SCIPPlan with and without the valid 
inequalities over all instances. In figure~\ref{fig:pairwise_timing}, 
each data point represents an instance that corresponds to a unique 
combination of the parameters provided in Table~\ref{tab:params}. The 
inspection of figure~\ref{fig:pairwise_timing} clearly highlights 
the benefit of including the valid inequalities as a part of the 
planning problems. On average, we observe that the addition 
of the valid inequalities improves the runtime performance of 
SCIPlan by around one to two orders of magnitude. As a result, 
we report results on the version of the planning problems 
with the valid inequalities in the remaining of this section.

\begin{figure}
\centering
\includegraphics[width=\linewidth]{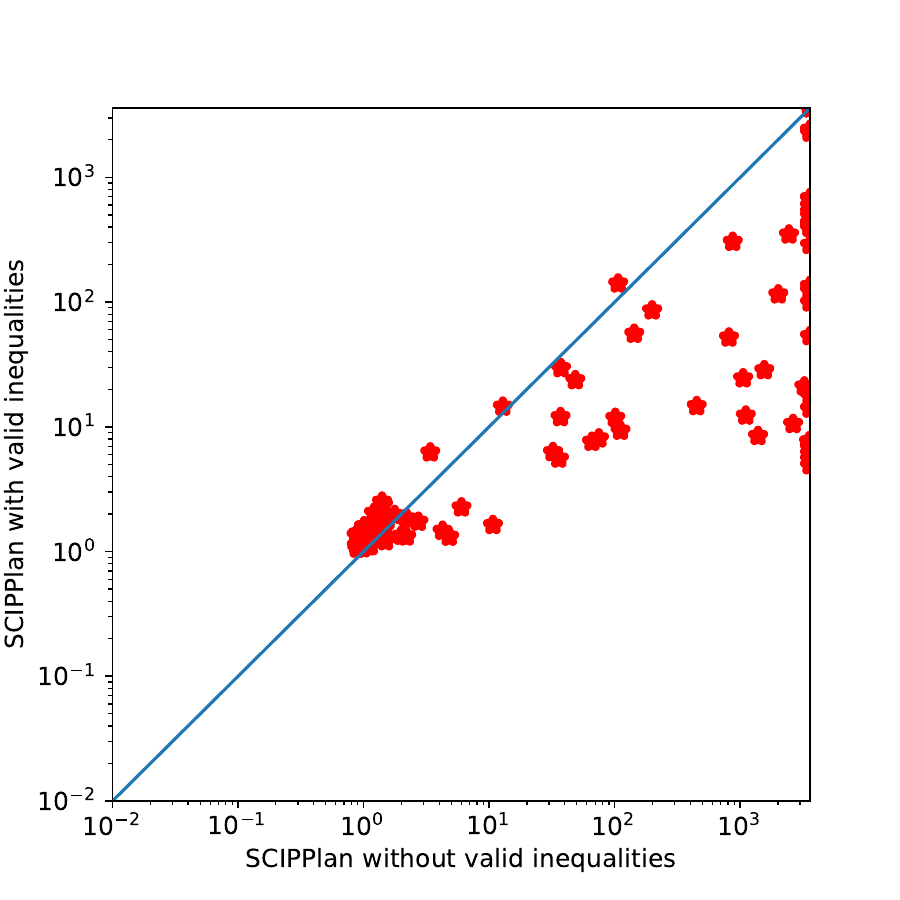}
\caption{Visualization of the effect of valid inequalities  
on the logarithmic runtime performance of SCIPPlan in seconds 
where each data point corresponds to an instance. The addition 
of the valid inequalities improves the runtime performance of 
SCIPlan by around one to two orders of magnitude on average.
}
\label{fig:pairwise_timing}
\end{figure}

\subsection{Problem Coverage}

We proceed with the analysis of our experiments with an overview of the 
results. Table~\ref{tab:coverage} summarizes the effect of using SCIPPlan 
to perform planning over both settings with different values of 
the parameters. Specifically, Table~\ref{tab:coverage} groups the 
total number of instances solved within the time limit for each problem 
setting and the values of parameters $b_1$, $b_2$, $c$ and $K$, over 
parameters $H$ and $I$. Overall, we observe that SCIPPlan had the highest 
success of covering the instances of Problem 1 (i.e., covering all instances). 
In Problem 2, we observe that SCIPPlan had the highest success of covering 
the instances with the parameter values $b_1=0.25$, $b_2=0.15$ and $c=0.2$ 
(i.e., 23 out of 24 instances), medium success with the parameter values 
$b_1=0.20$, $b_2=0.10$, $c=0.15$, $K=250$ (i.e., 9 out of 12 instances), 
and the lowest success with the parameter values $b_1=0.20$, $b_2=0.10$, 
$c=0.15$, $K=200$ (i.e., 4 out of 12 instances). Next, we will inspect both the runtime performance and the normalized 
solution quality of SCIPPlan over both problem settings.

\begin{table}
\scriptsize
  \centering
    \begin{tabular}{| c | c | c |}
    \hline
    Problem & Parameters & Coverage \\ \hline
    1 & $b_1=0.20$, $b_2=0.10$, $c=0.15$, $K=200$ & 12/12 \\ \hline
    1 & $b_1=0.20$, $b_2=0.10$, $c=0.15$, $K=250$ & 12/12 \\ \hline
    1 & $b_1=0.25$, $b_2=0.15$, $c=0.2$, $K=200$ & 12/12 \\ \hline
    1 & $b_1=0.25$, $b_2=0.15$, $c=0.2$, $K=250$ & 12/12 \\ \hline
    2 & $b_1=0.20$, $b_2=0.10$, $c=0.15$, $K=200$ & 4/12 \\ \hline
    2 & $b_1=0.20$, $b_2=0.10$, $c=0.15$, $K=250$ & 9/12 \\ \hline
    2 & $b_1=0.25$, $b_2=0.15$, $c=0.2$, $K=200$ & 11/12 \\ \hline
    2 & $b_1=0.25$, $b_2=0.15$, $c=0.2$, $K=250$ & 12/12 \\ \hline
    \multicolumn{2}{| c |}{Total} & 84/96 \\ \hline
  \end{tabular}
  \caption{Summary of the problem coverage grouped 
by the problem settings and the values of parameters.}
  \label{tab:coverage}
\end{table}

\subsection{Runtime Performance}

The runtime performance of SCIPPlan over all unique instances 
is visualized by figure~\ref{fig:runtime} using heatmaps. The 
inspection of figure~\ref{fig:runtime} for Problem 1 reveals 
that decreasing the value of parameter $\delta$ typically 
decreases the runtime performance of SCIPPlan, since the lower 
values of $\delta$ correspond higher values of horizon $H$. The 
inspection of figure~\ref{fig:runtime} for Problem 2 reveals 
that decreasing the values of parameters $\delta$ and $K$ can 
significantly decrease the runtime performance of SCIPPlan, since 
the lower values of $\delta$ and $K$ typically correspond to the 
existence of solutions with higher values of horizon $H$. 
%Moreover, it reinforces our previous findings on the effect of parameters $b_1$, $b_2$ and $c$ on the runtime performance of SCIPPlan where we observe significant decreases in runtimes when the values of $b_1$, $b_2$ and $c$ are all increased by 0.05. 
Overall, we have 
not found a significant effect of the value of parameter $I$ on 
the runtime performance of SCIPPlan.

\begin{figure}  
    \centering
    \includegraphics[width=.49\linewidth]{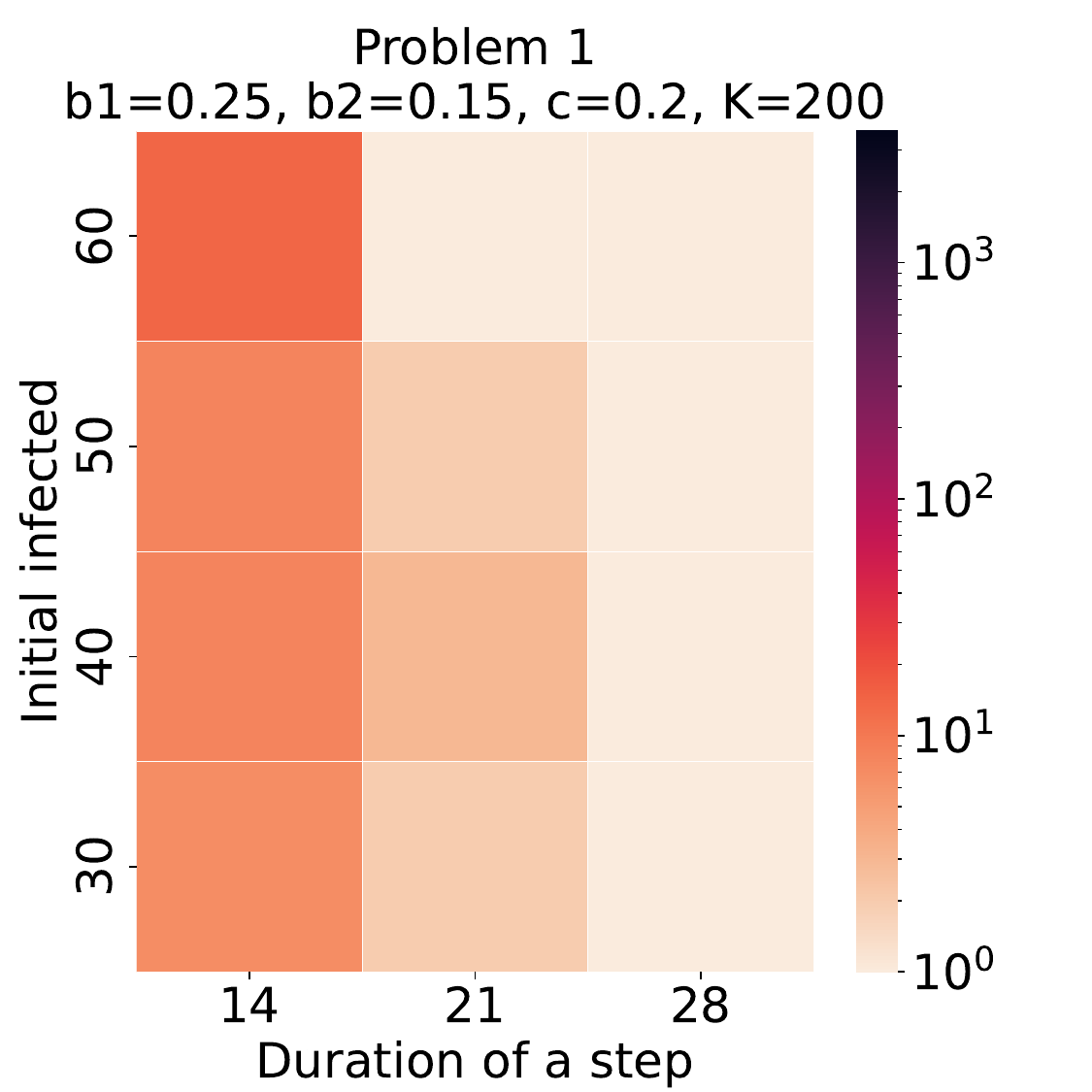}
    \label{fig:runtime1_0.25_0.15_0.2_200}
    \centering
    \includegraphics[width=.49\linewidth]{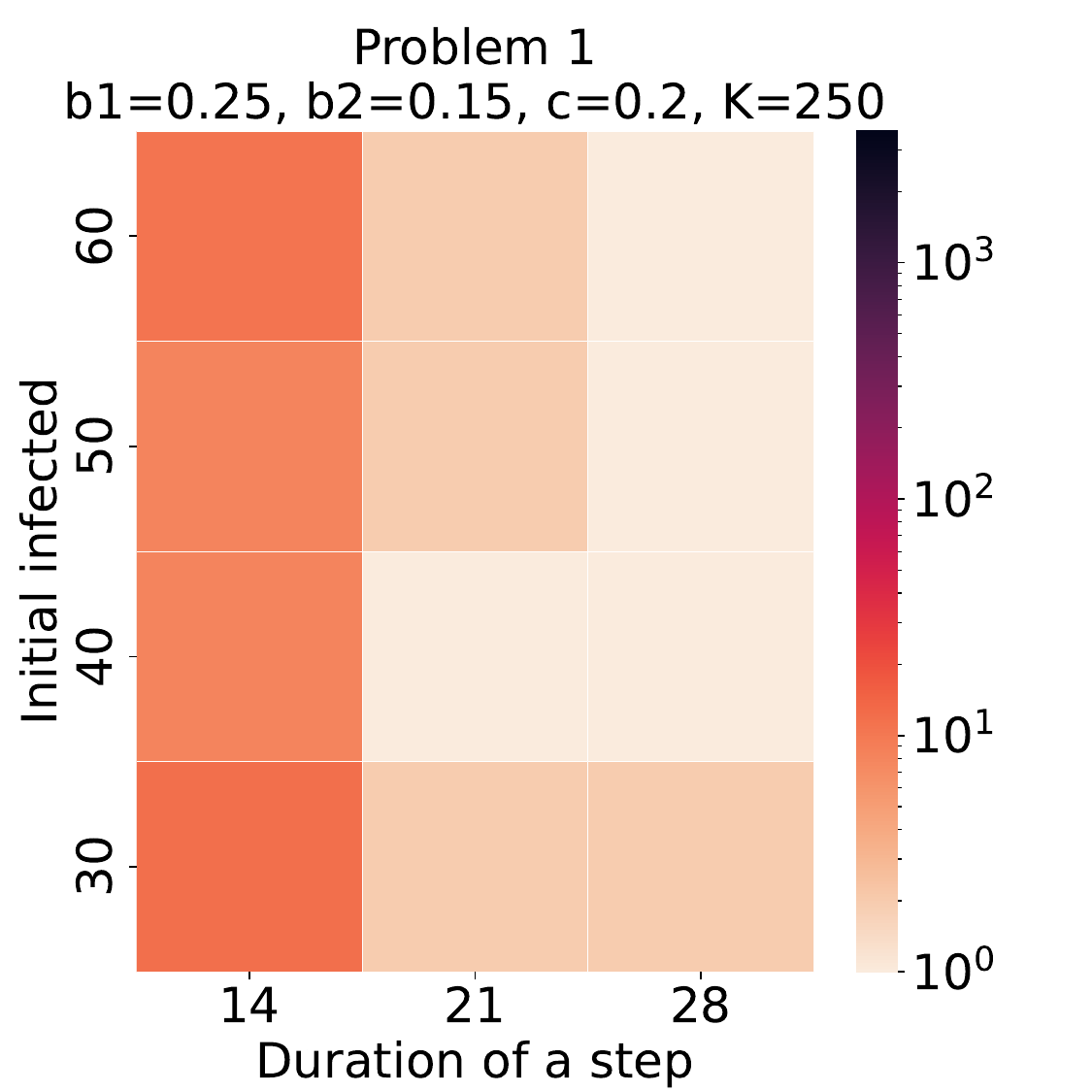}
    \label{fig:runtime1_0.25_0.15_0.2_250}

    \centering
    \includegraphics[width=.49\linewidth]{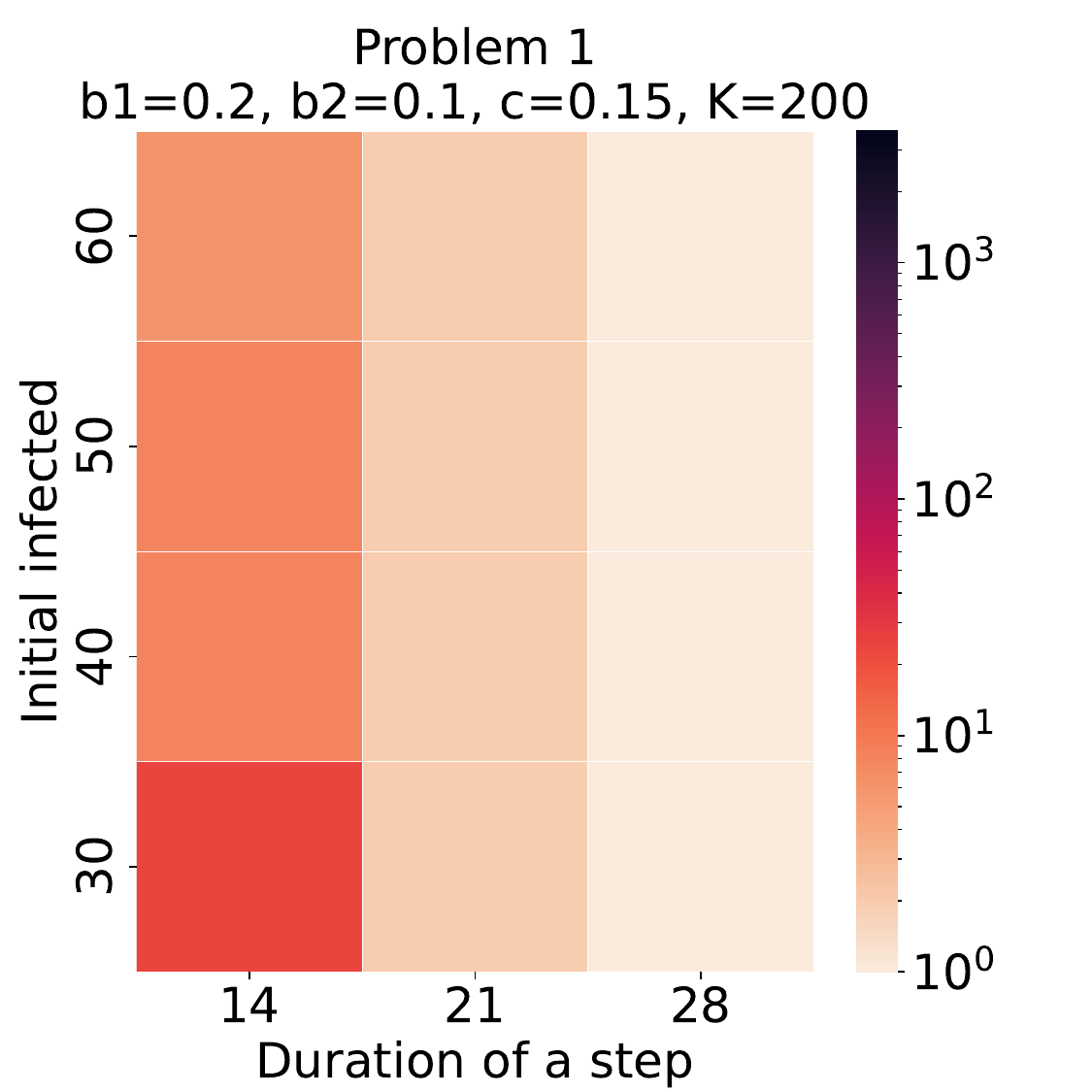}
    \label{fig:runtime1_0.2_0.1_0.15_200}
    \centering
    \includegraphics[width=.49\linewidth]{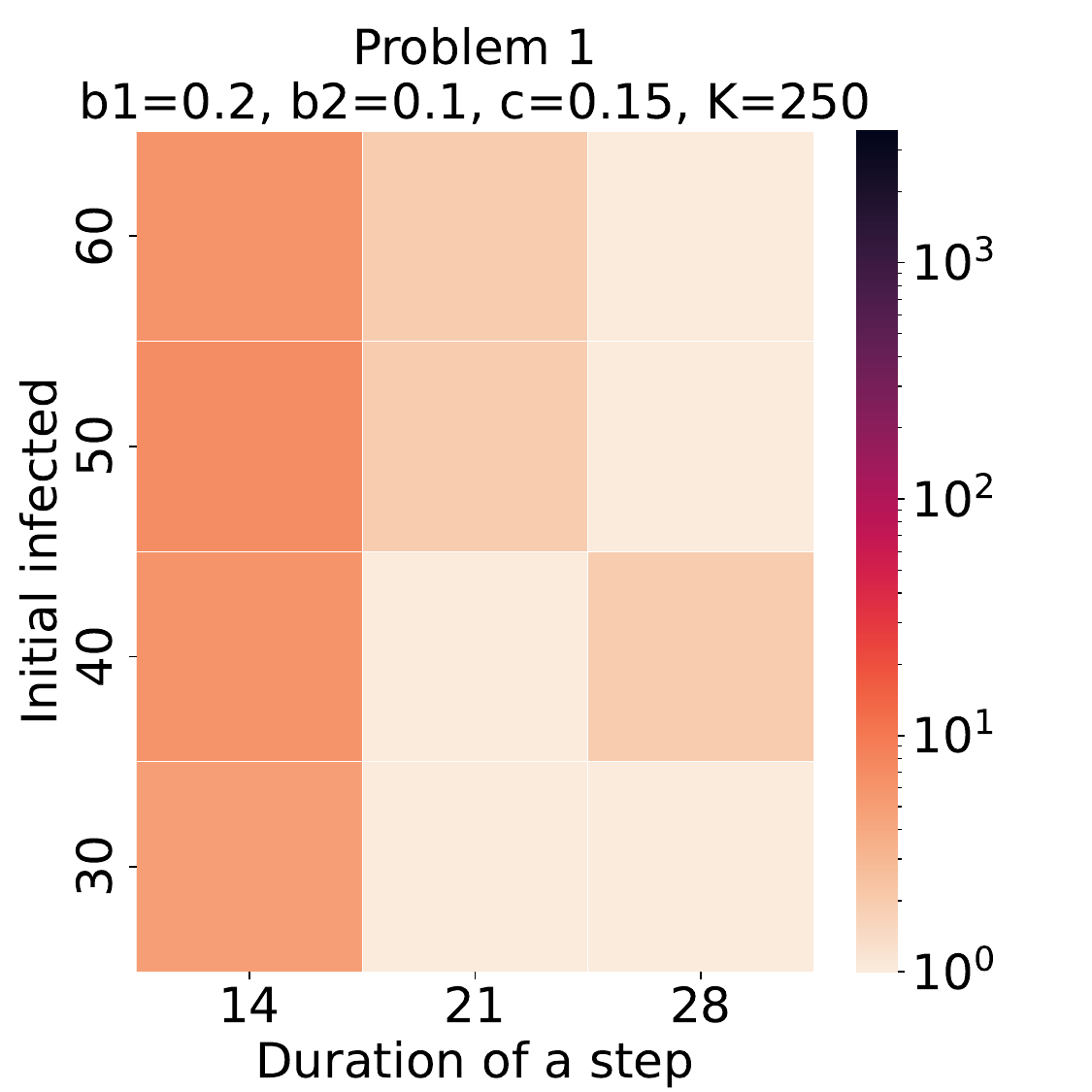}
    \label{fig:runtime1_0.2_0.1_0.15_250}

    \centering
    \includegraphics[width=.49\linewidth]{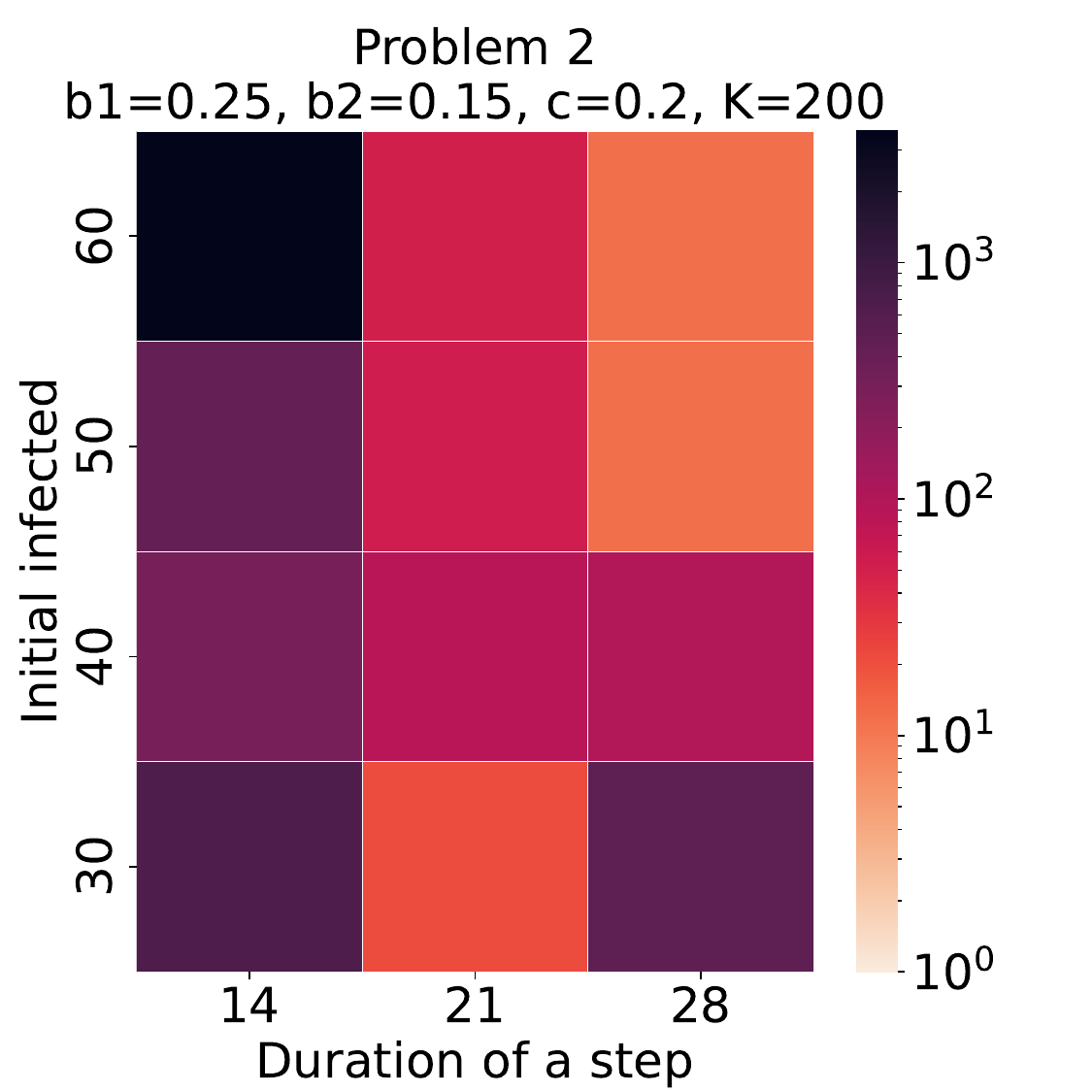}
    \label{fig:runtime2_0.25_0.15_0.2_200}
    \centering
    \includegraphics[width=.49\linewidth]{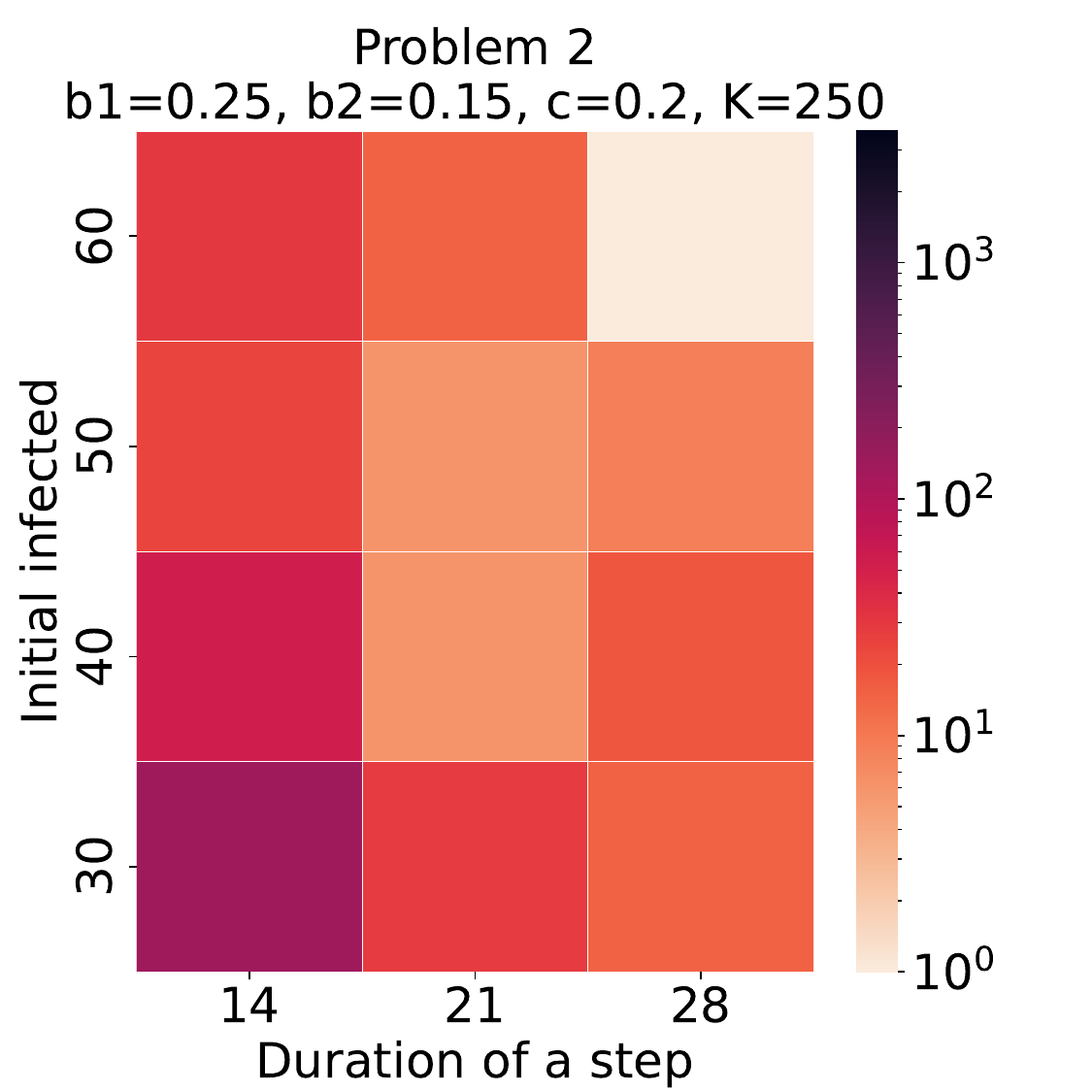}
    \label{fig:runtime2_0.25_0.15_0.2_250}

    \centering
    \includegraphics[width=.49\linewidth]{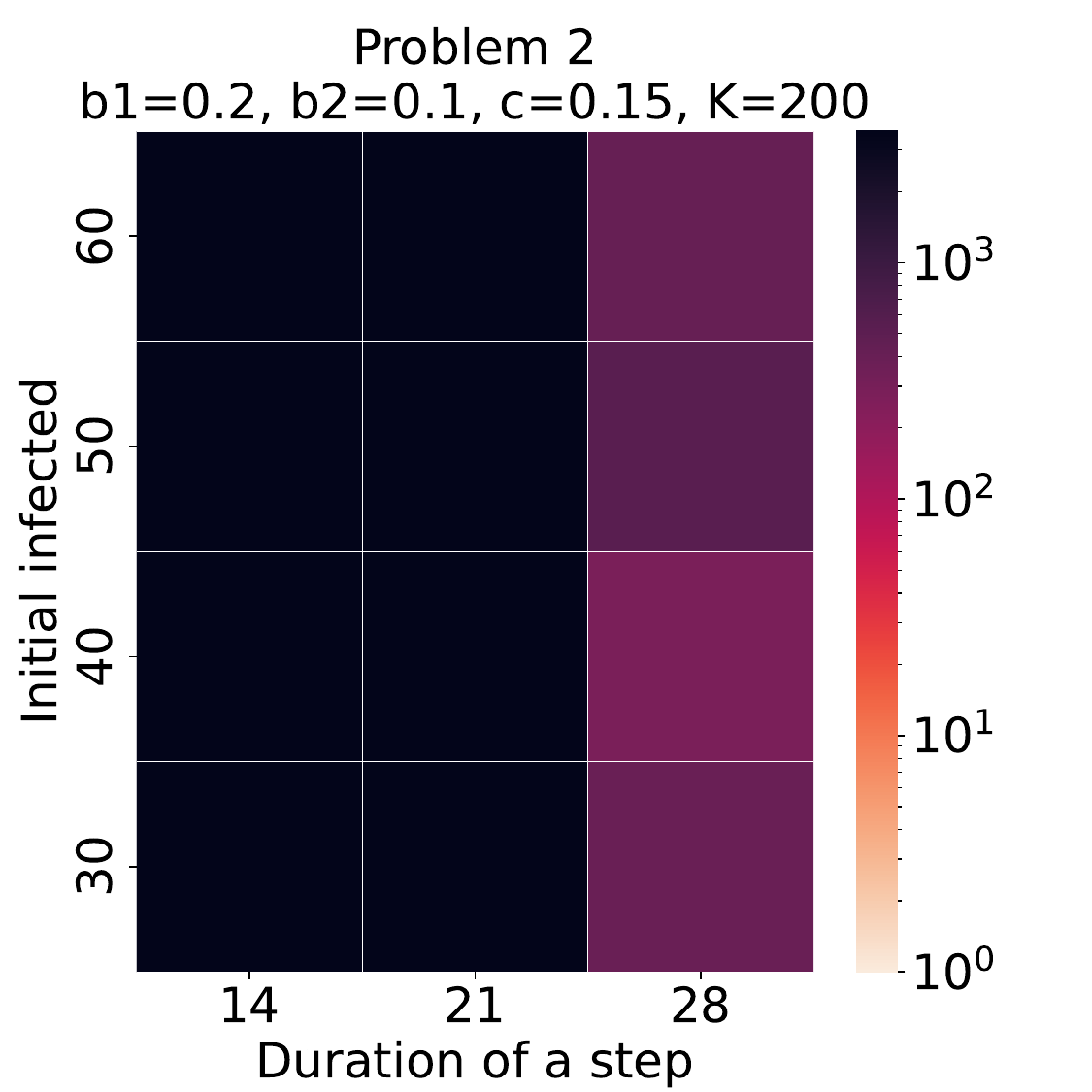}
    \label{fig:runtime2_0.2_0.1_0.15_200}
    \centering
    \includegraphics[width=.49\linewidth]{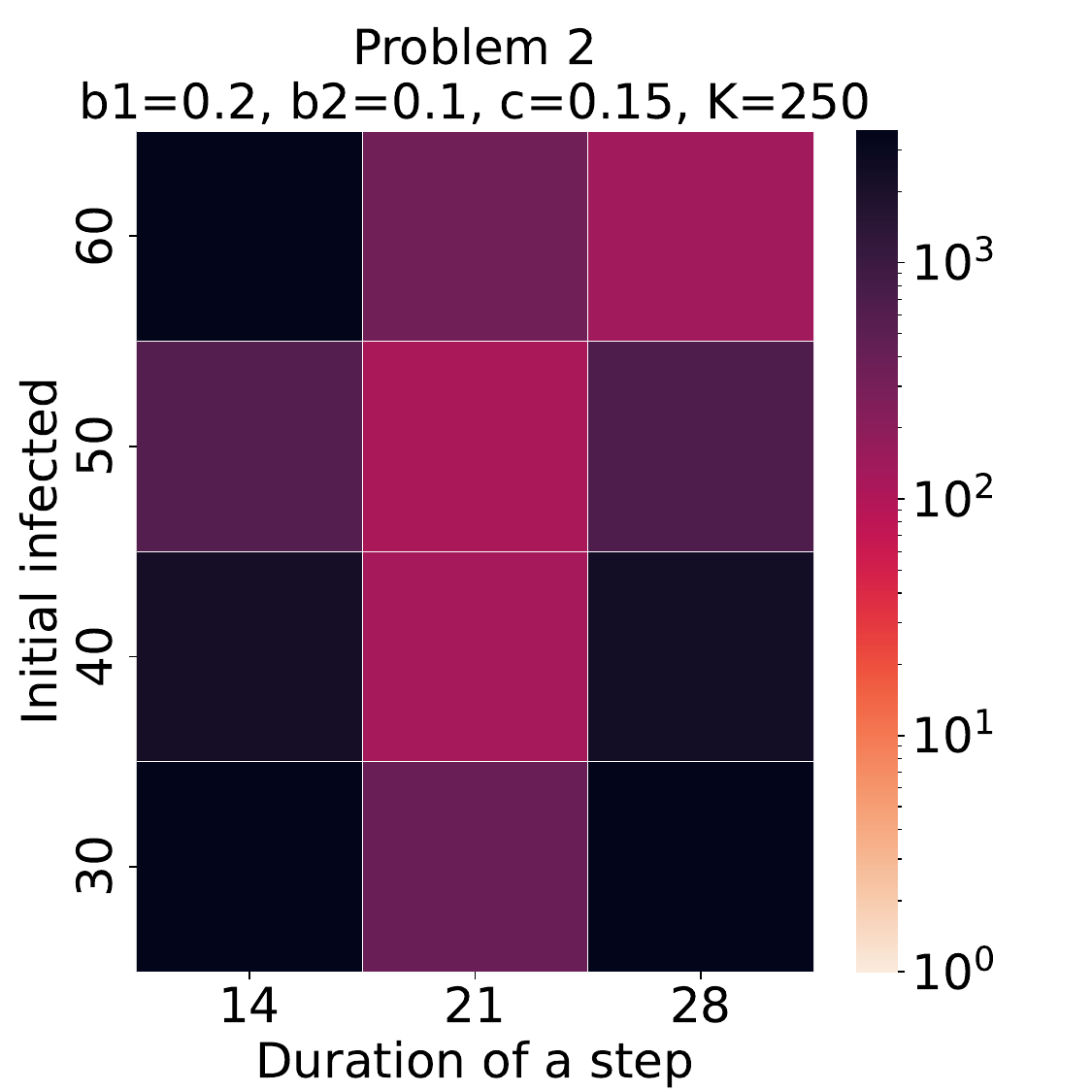}
    \label{fig:runtime2_0.2_0.1_0.15_250}
\caption{Visualization of logarithmic runtime performance in seconds 
for both problem settings. Overall, decreasing the value of parameter 
$\delta$ typically decreases the runtime performance of SCIPPlan, 
since the lower values of $\delta$ correspond higher values of horizon $H$.}
\label{fig:runtime}
\end{figure}

\subsection{Normalized Solution Quality}

The normalized solution quality of SCIPPlan over all unique instances, 
that is defined as the total action duration 
$\delta \sum_{t=1}^{H} a^t_1$, is visualized by figure~\ref{fig:quality}. 
The inspection of figure~\ref{fig:quality} highlights the clear benefit 
of using smaller values of parameter $\delta$ which allows for a more 
granular control of the population. In Problem 2, decreasing the 
value of $\delta$ by one week decreases the total action duration by around 
26 days (i.e., on average). Moreover, we observed that 
increasing the value of parameter $K$ can significantly decrease the total 
action duration. Overall, we have not found a significant effect of the 
value of parameter $I$ on the total action duration.

\begin{figure}
    \centering
    \includegraphics[width=.49\linewidth]{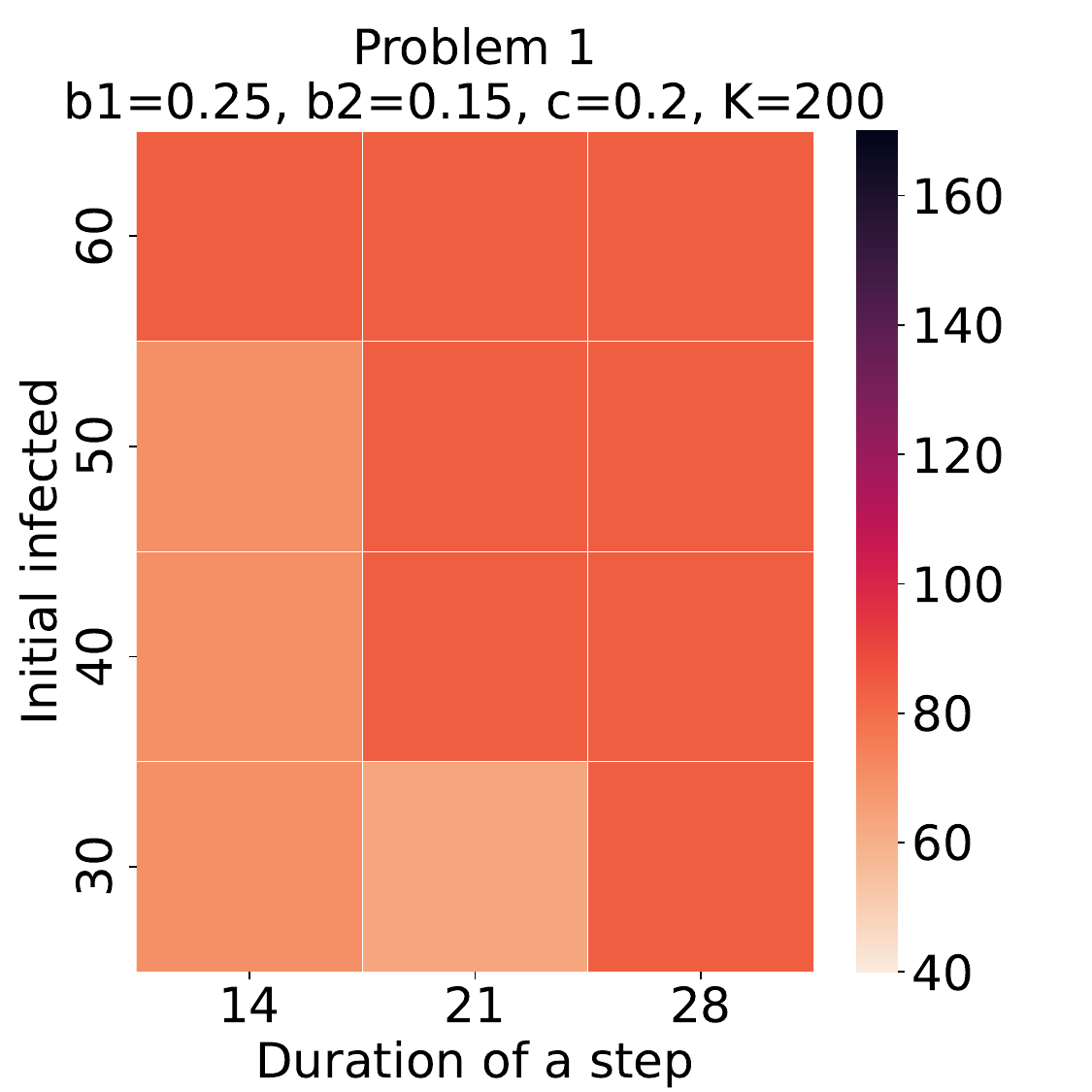}
    \label{fig:quality1_0.25_0.15_0.2_200}
    \centering
    \includegraphics[width=.49\linewidth]{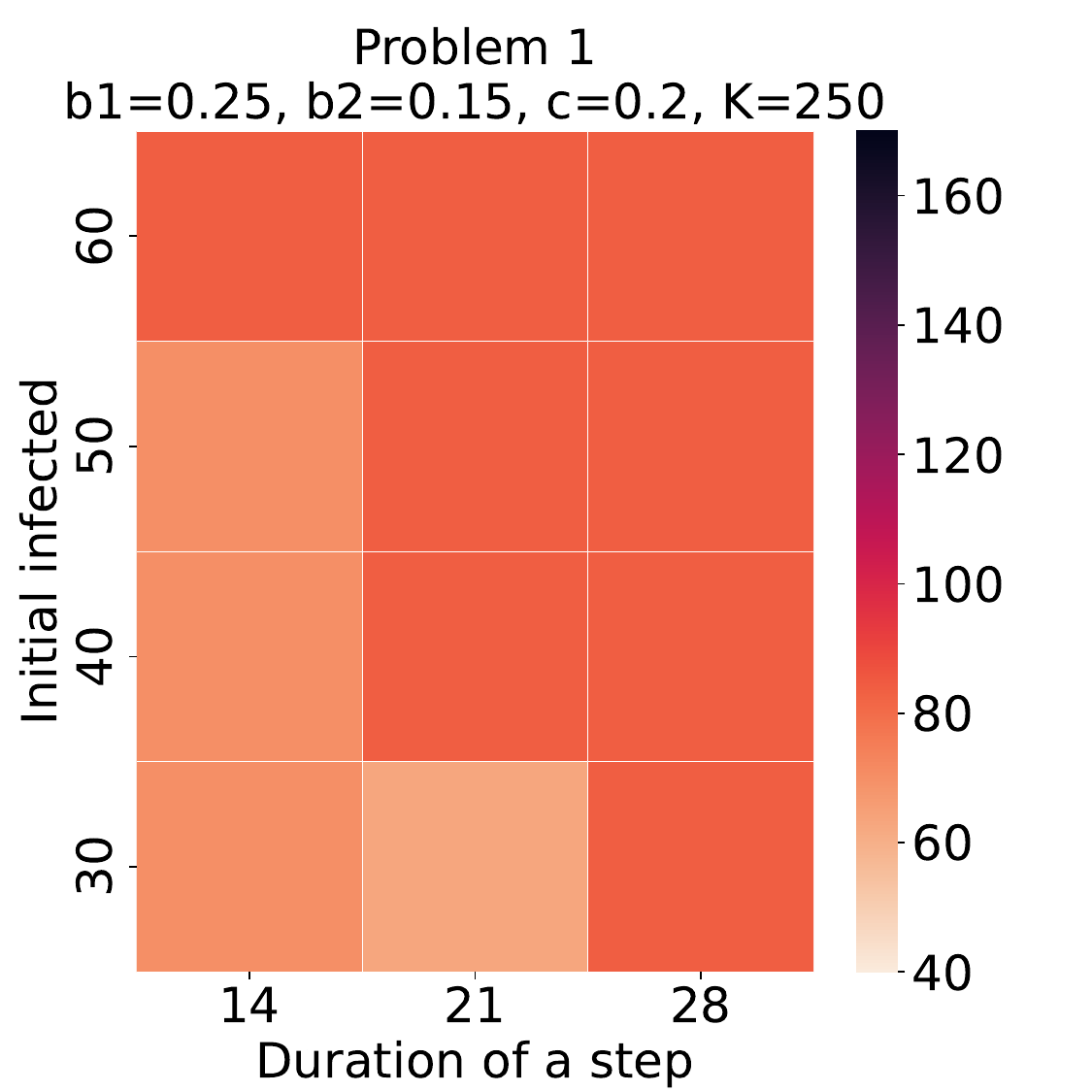}
    \label{fig:quality1_0.25_0.15_0.2_250}

    \centering
    \includegraphics[width=.49\linewidth]{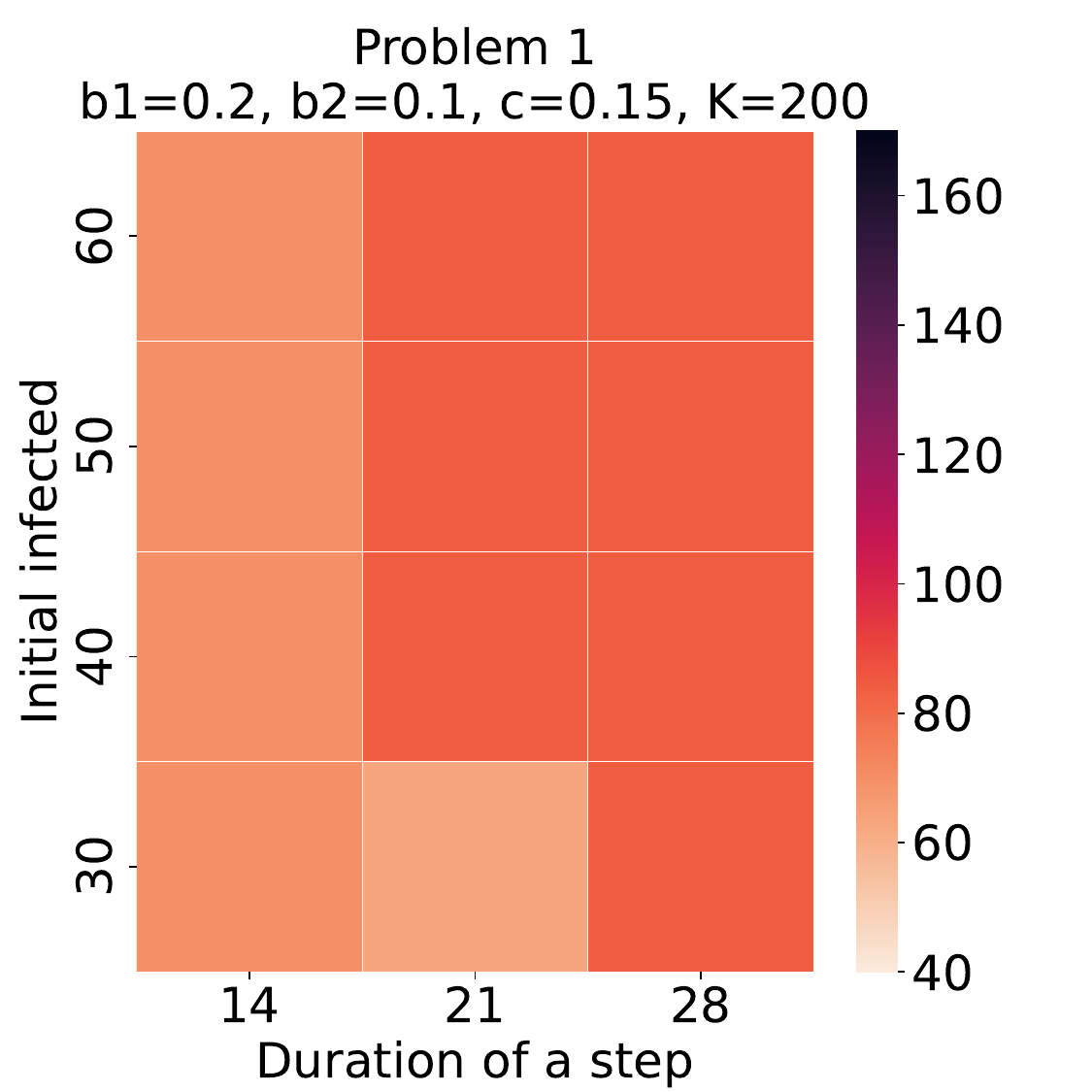}
    \label{fig:quality1_0.2_0.1_0.15_200}
    \centering
    \includegraphics[width=.49\linewidth]{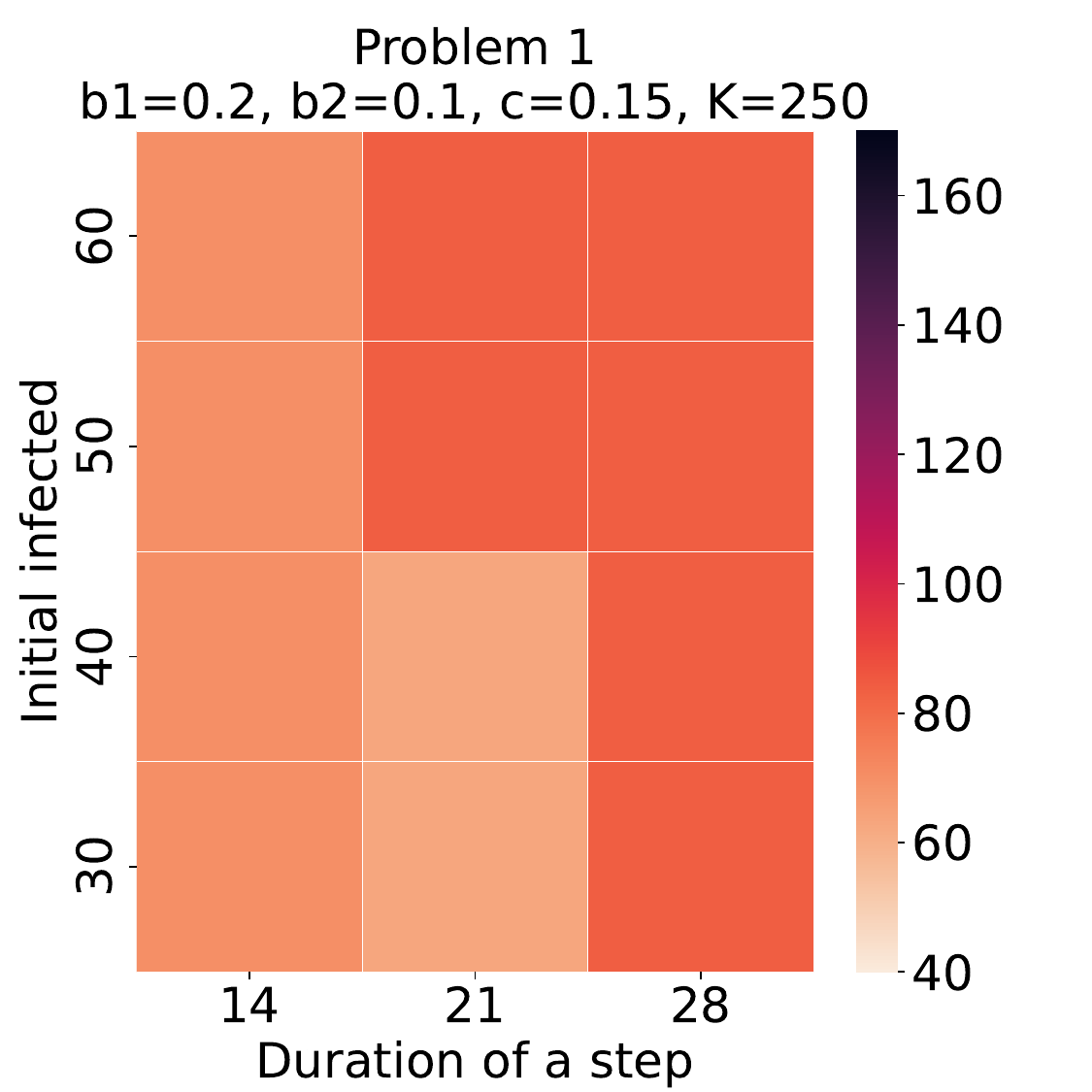}
    \label{fig:quality1_0.2_0.1_0.15_250}

    \centering
    \includegraphics[width=.49\linewidth]{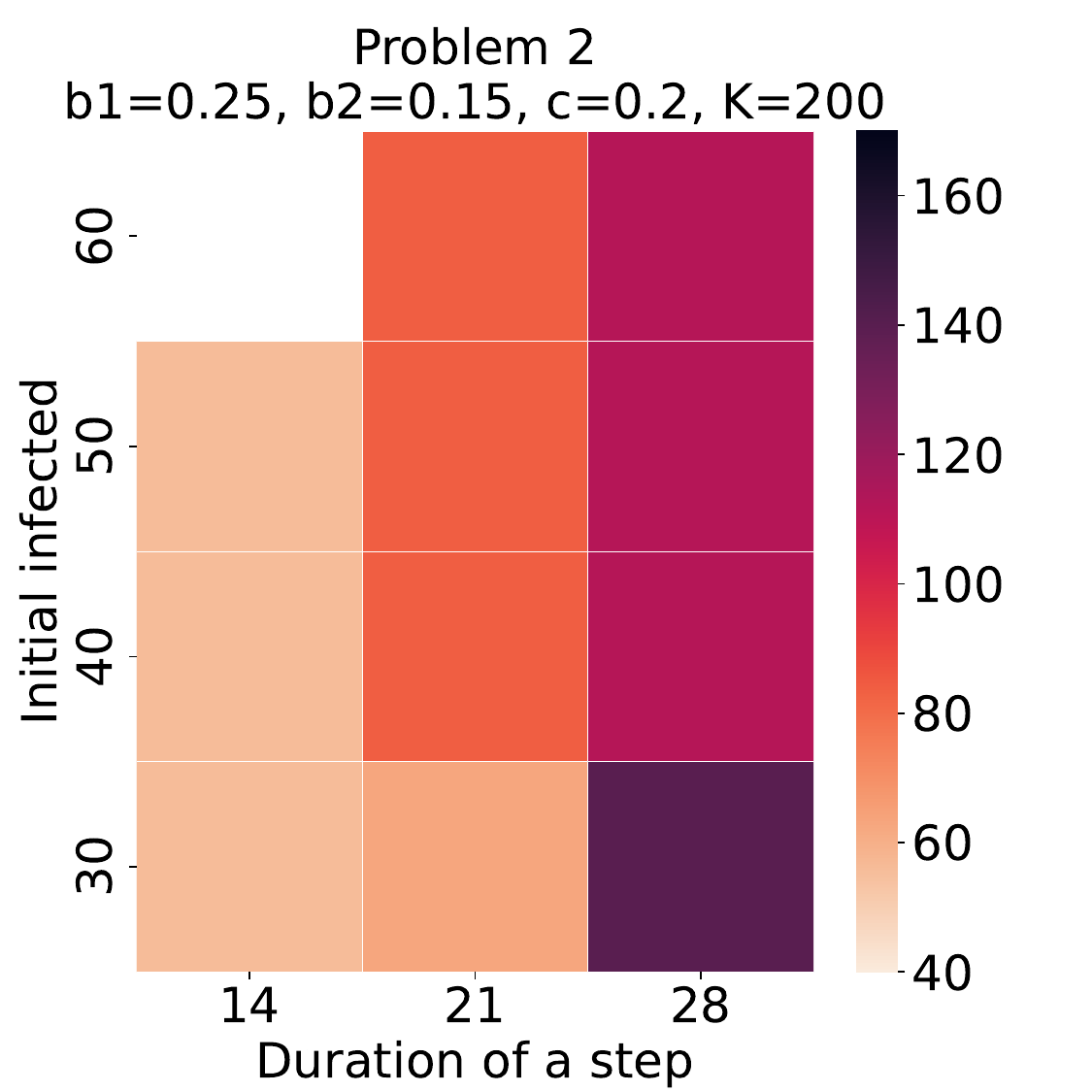}
    \label{fig:quality2_0.25_0.15_0.2_200}
    \centering
    \includegraphics[width=.49\linewidth]{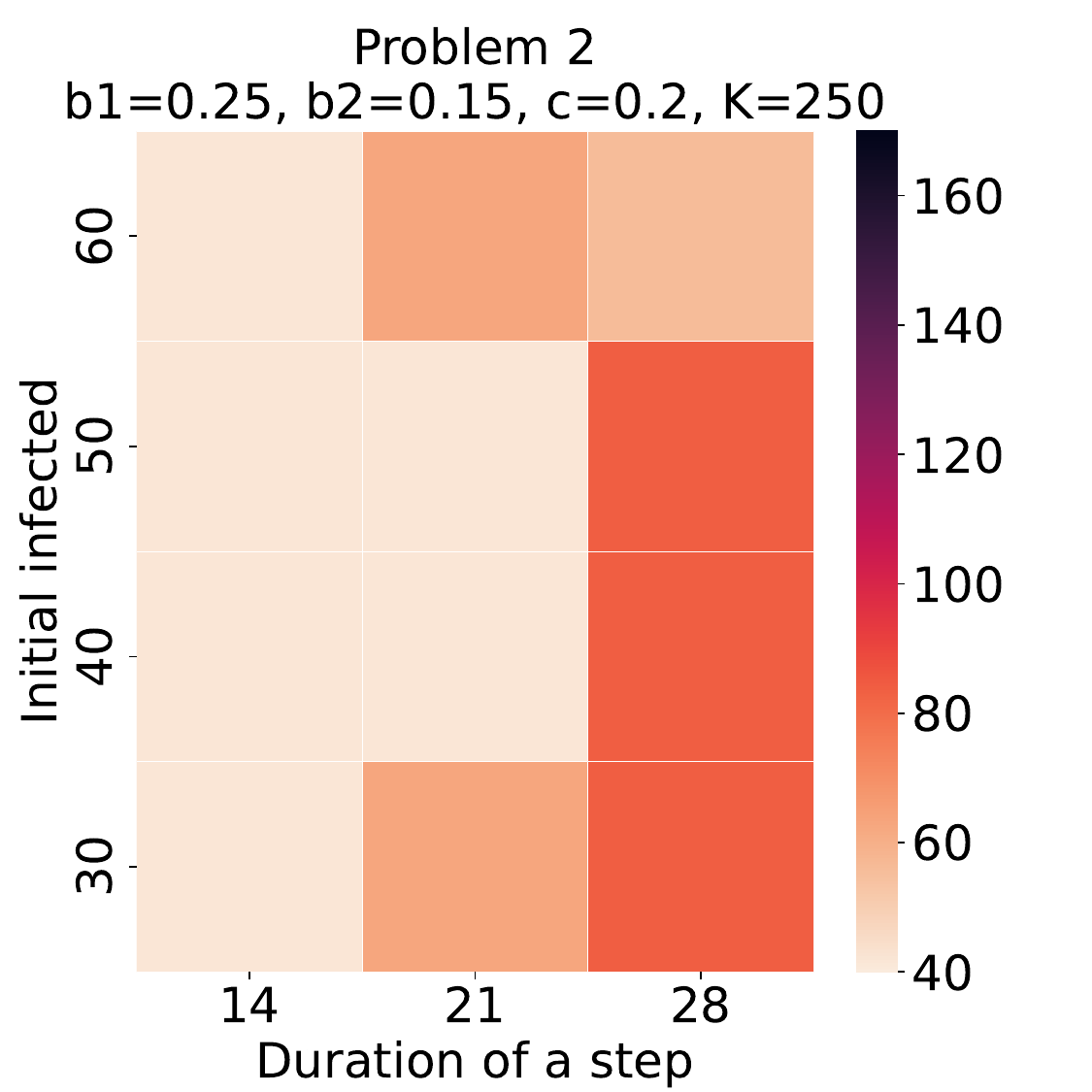}
    \label{fig:quality2_0.25_0.15_0.2_250}

    \centering
    \includegraphics[width=.49\linewidth]{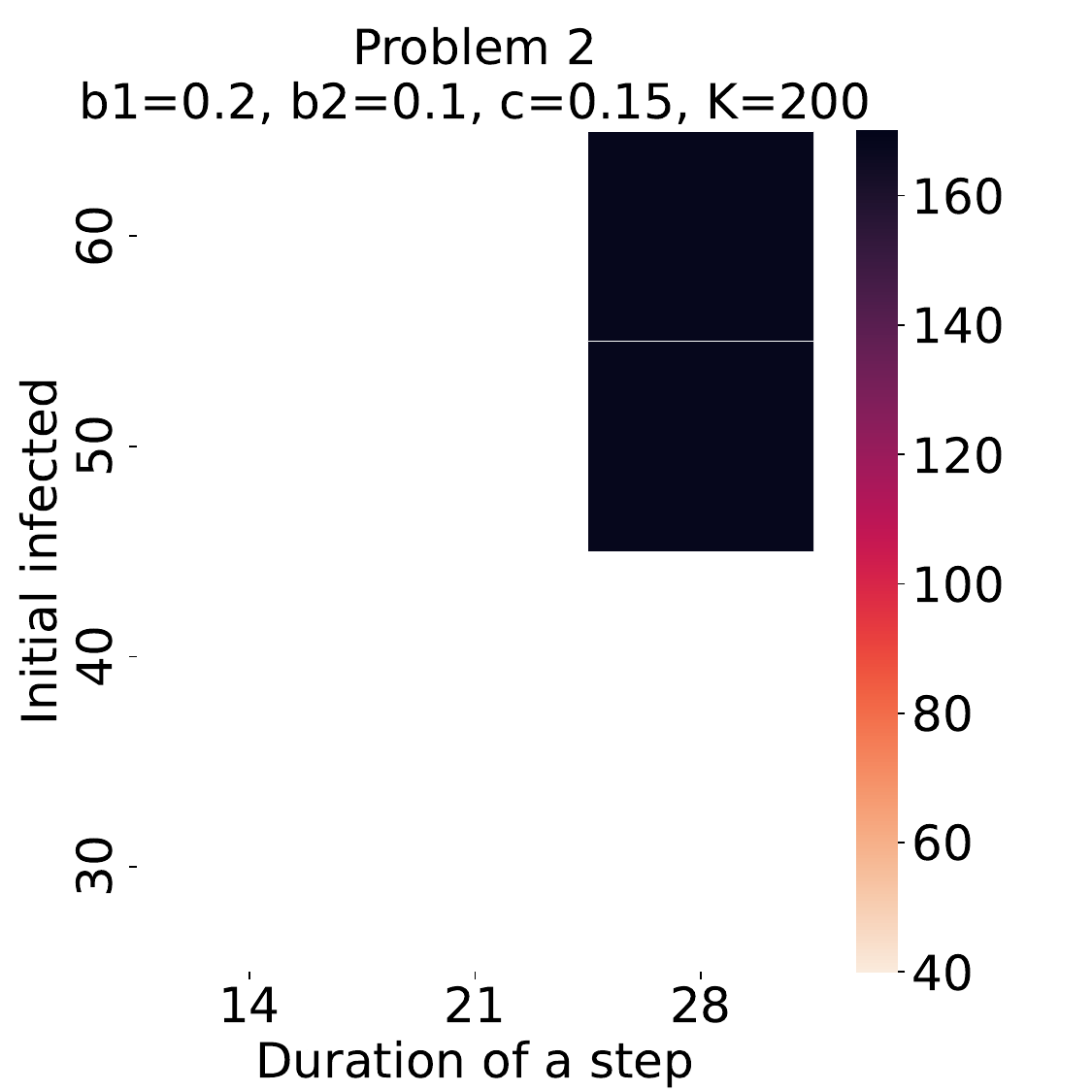}
    \label{fig:quality2_0.2_0.1_0.15_200}
    \centering
    \includegraphics[width=.49\linewidth]{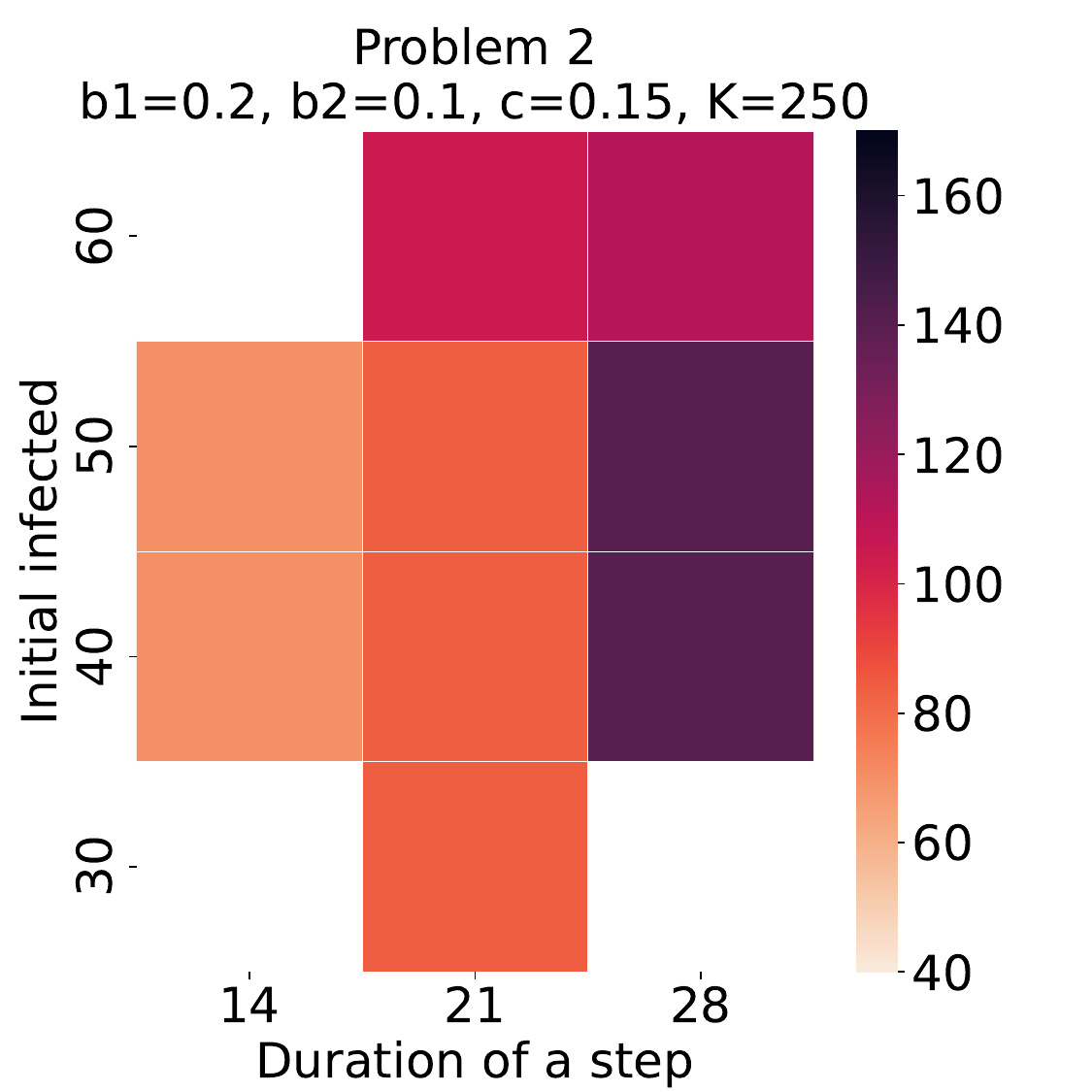}
    \label{fig:quality2_0.2_0.1_0.15_250}
\caption{Visualization of total action durations in days for both problem settings. 
Overall, using smaller values of parameter $\delta$ increases the solution quality 
and allows for a more granular control of the population.}
\label{fig:quality}
\end{figure}

\subsection{Effect of Using Variable Step Duration}

In this section, we analyze the effect of using variable step 
duration over fixed step duration on the runtime performance 
and the solution quality of SCIPPlan. In order to achieve this, 
we ran SCIPPlan with both fixed step duration and variable step 
duration over all instances of Problem 1.

Figure~\ref{fig:quality_2} visualizes the total action duration 
$\sum_{t=1}^{H} \Delta^t a^t_1$ comparison between using variable 
step duration and fixed step duration. The inspection of the 
figures highlights the clear benefit of using variable step 
duration over fixed step duration where SCIPPlan finds solutions 
with higher solution quality (i.e., lower total action duration).
However, we have found that the increase in solution quality 
comes at a price. Figure~\ref{fig:pairwise_timing_2} visualizes 
the logarithmic runtime comparison between using variable step 
duration and fixed step duration. On average, we observe that 
the use of variable step size decreases the runtime performance 
of SCIPlan by around two orders of magnitude, which highlights 
the additional computational resources required to improve the 
solution quality using variable step duration.

\begin{figure}
    \centering
    \includegraphics[width=.49\linewidth]{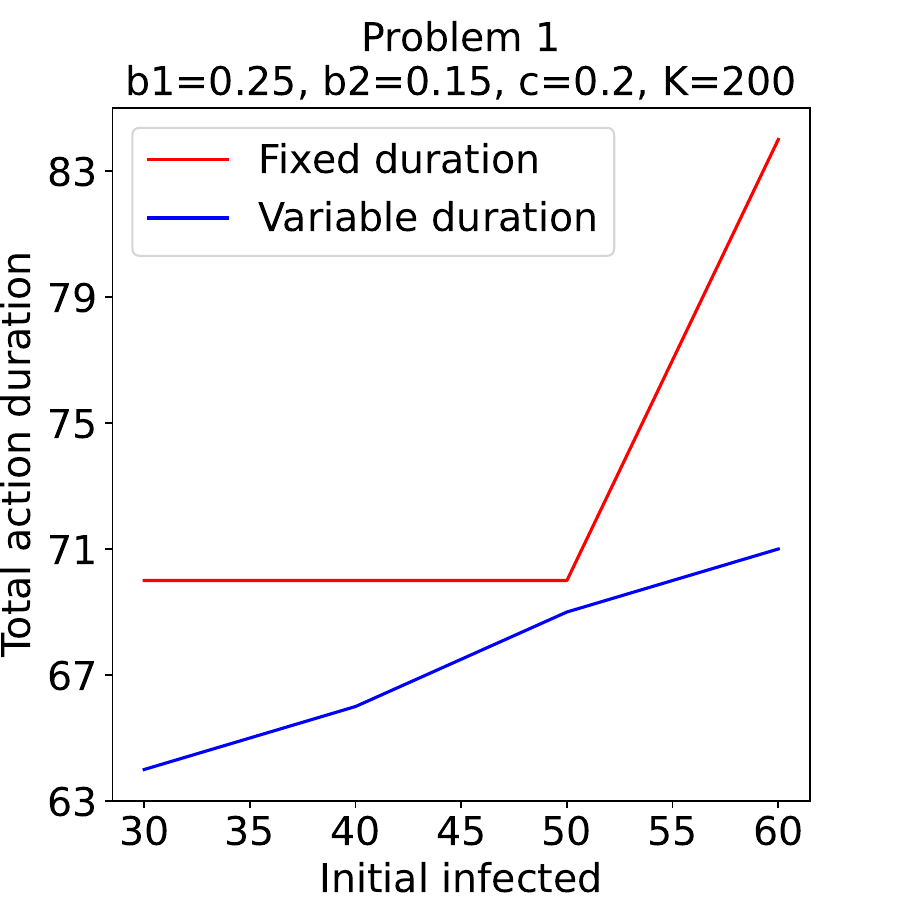}
    \label{fig:pairwise_reward_2}
    \centering
    \includegraphics[width=.49\linewidth]{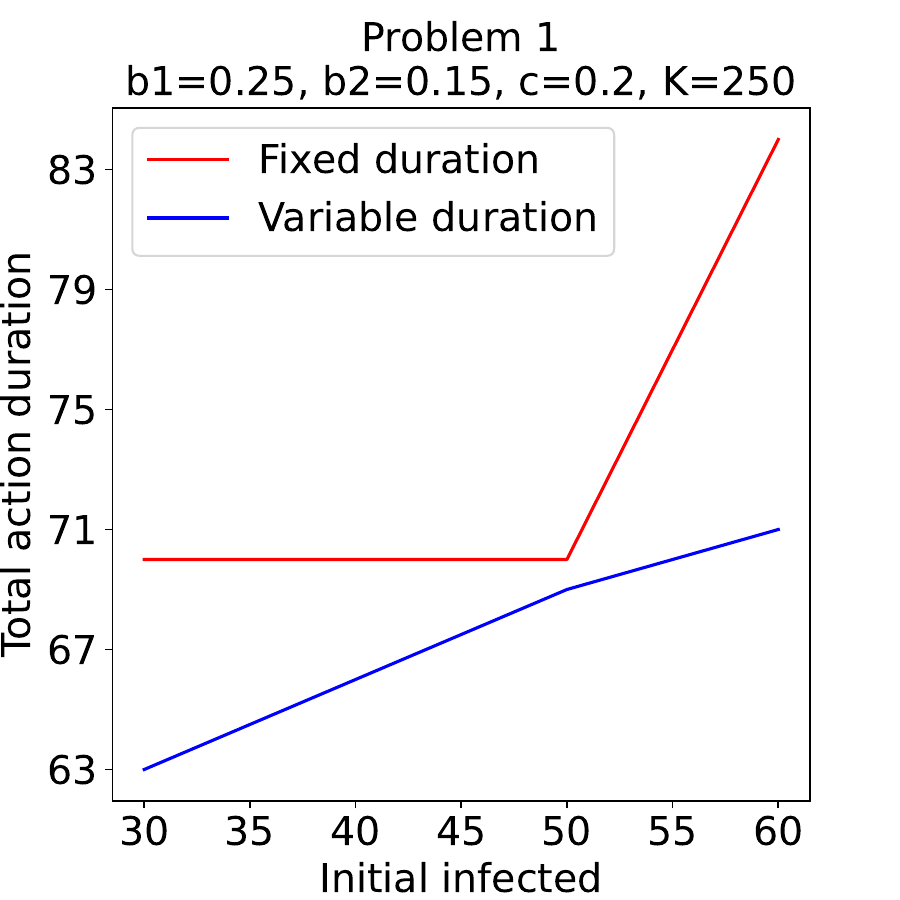}
    \label{fig:pairwise_reward_4}
    
    \centering
    \includegraphics[width=.49\linewidth]{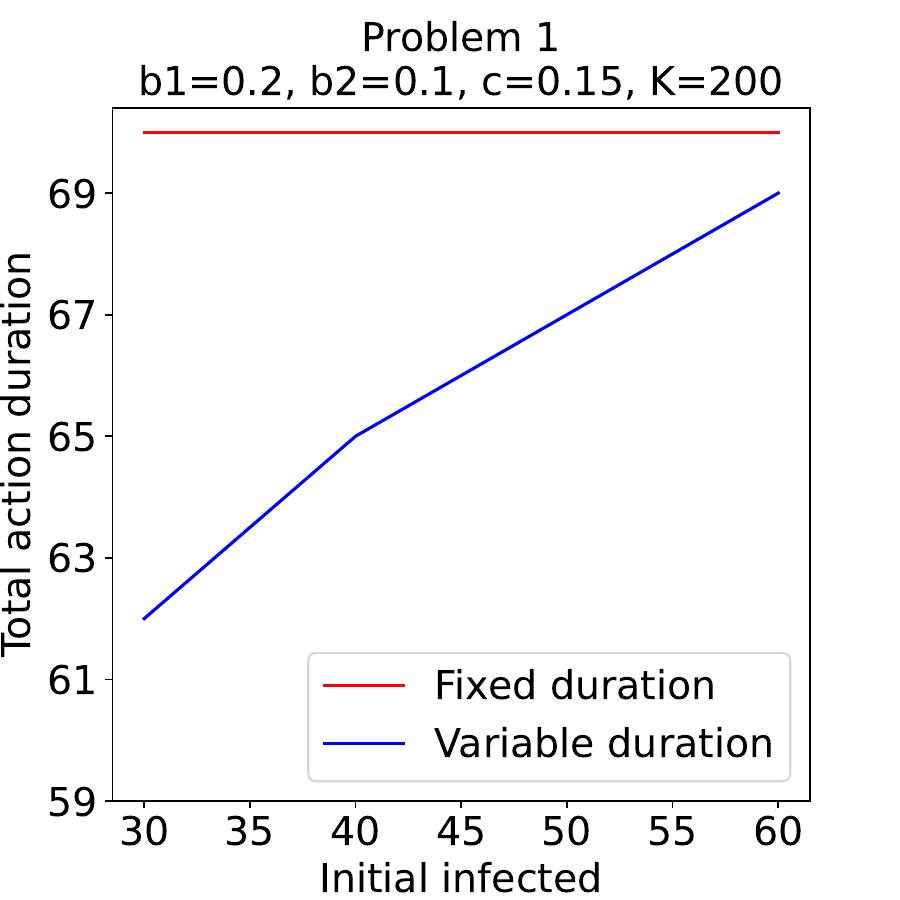}
    \label{fig:pairwise_reward_1}
    \centering
    \includegraphics[width=.49\linewidth]{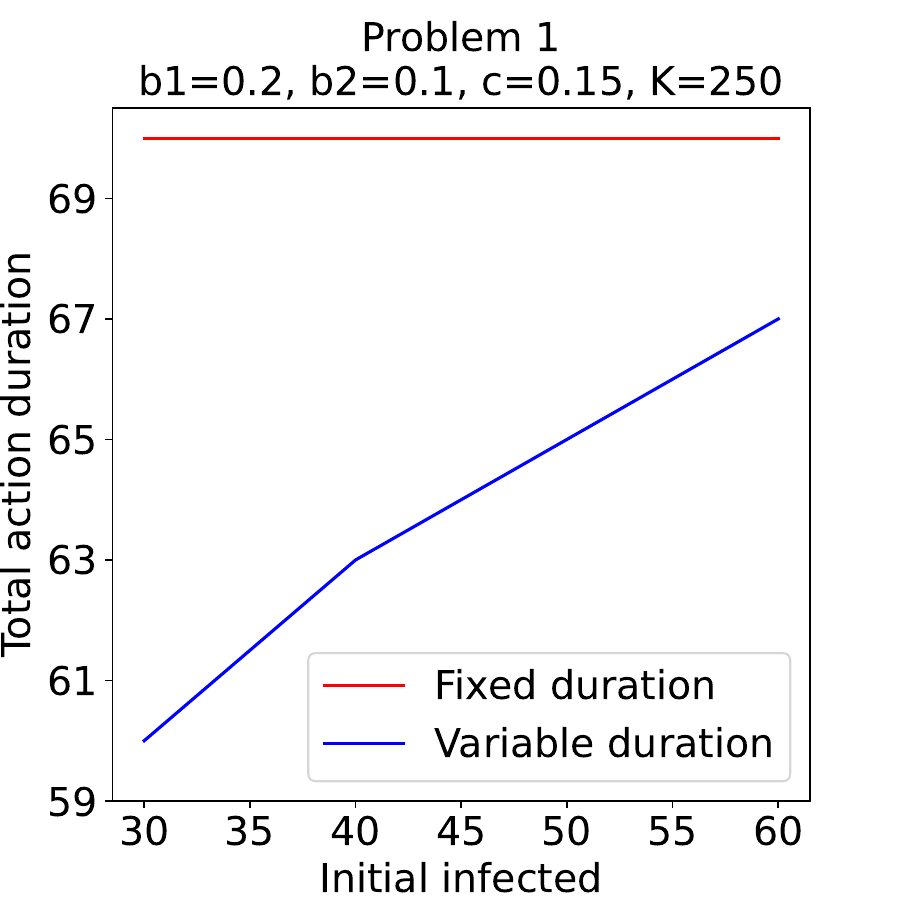}
    \label{fig:pairwise_reward_3}
\caption{Visualization of total action durations in days over 
different values of initial infected. The use of the variable 
step duration over fixed step duration results in higher solution 
quality (i.e., lower total action duration) in all instances.}
\label{fig:quality_2}
\end{figure}

\begin{figure}
\centering
\includegraphics[width=\linewidth]{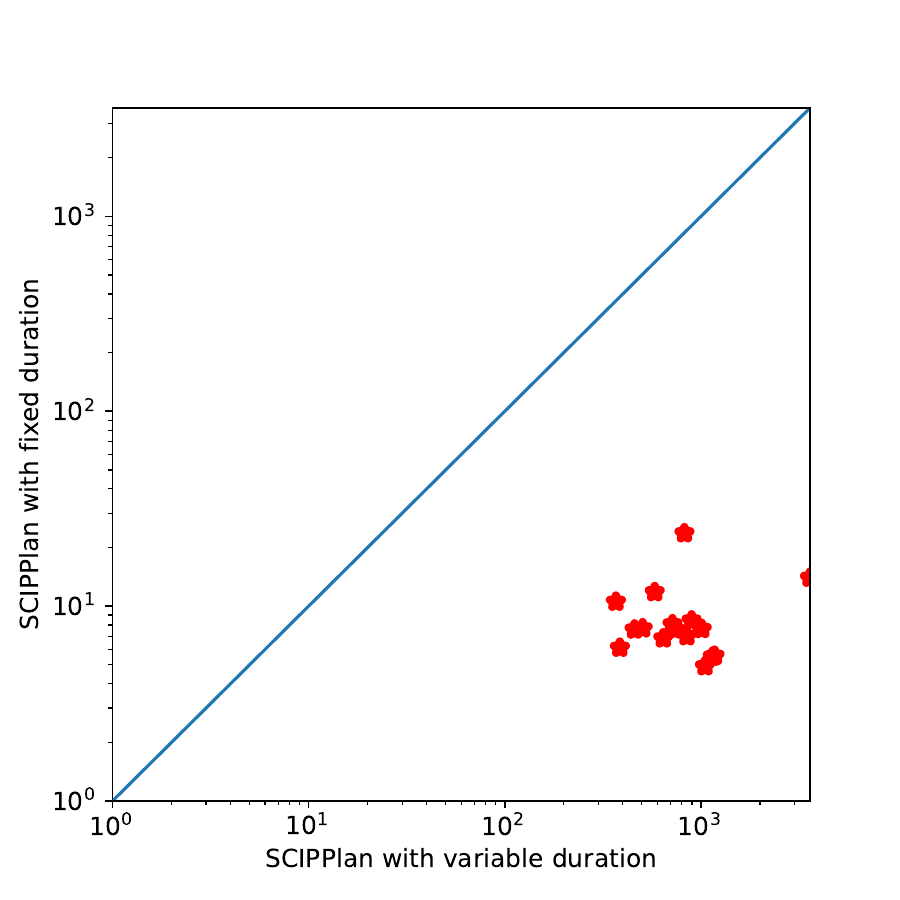}
\caption{Visualization of the effect of using variable step 
duration on the logarithmic runtime performance of SCIPPlan 
in seconds where each data point corresponds to an 
instance. The use of the variable step duration decreases 
the runtime performance of SCIPlan by around two orders of 
magnitude on average.}
\label{fig:pairwise_timing_2}
\end{figure}

\section{Discussion and Related Work}
\label{sec:discussion}

In this section, we discuss the importance of our theoretical 
and experimental results in relation to the literature 
for the purpose of opening new areas for future work.

%In section~\ref{sec:assumptions}, we have made some assumptions which led to the pandemic planning problem we have solved in section~\ref{sec:experiments}. It is important to note that if the mortality rate $q\in [0,1]$ of the disease is known and its relative cost is considered significant, we can run SCIPPlan over a fixed horizon $H$ (i.e., instead of increasing values of $H$) and modify the goal constraint (\ref{scip12_pand}) to: 
%\begin{align}
%&s^{H+1}_{3} \leq q N\label{scip_alternative}
%\end{align}
%in order to control the pandemic until either lower values of infection rates $b_1$ and $b_2$ are observed and/or alternative actions (e.g., vaccinations) become available.

In the Theoretical Results section, we have shown theoretical results on 
the finiteness and solution quality of SCIPPlan for solving the 
pandemic planning problem. Our theoretical results yield
strong guarantees on the solutions provided by SCIPPlan such 
that the temporal constraint is satisfied with relatively small 
values of $\gamma$ where $\gamma$ is 1.00 for $K=200$ and $\gamma$ 
is $1.25$ for $K=250$.

In the Experimental Results section, we have shown the computational 
performance of our approach to control a pandemic and other related 
systems. Our experimental results demonstrate the computational 
benefit of deriving valid inequalities for the planning problem. 
Some of these valid inequalities (e.g., monotonicity of state variables) 
are common in many other continuous time decision making problems where 
the state evolution is governed by a system of partial differential 
equations. For future work, we plan to investigate the automation 
of the process of valid inequality derivation for similar metric 
hybrid planning problems. Further, our experimental results also 
demonstrate the computational viability of using the exact 
solution equations of the SIR model (i.e., equations (\ref{sus}-
\ref{rem})) to construct a nonlinear state transition function $T$ 
using both exponential and logarithmic expressions, and perform 
formal reasoning over continuous time. Our experimental results 
compliment the success of the existing decision making systems 
that allow for continuous time decision making under additional 
restrictive assumptions (e.g., the assumption that the state 
transition function $T$ is piecewise 
linear~\cite{Shin2005,Coles2012,Chen2021}, 
polynomial~\cite{Cashmore2016} etc.). Finally, we highlight the 
importance of our approach given the availability of machine 
learning techniques~\cite{Raissi2019,Karniadakis2021} 
that can successfully approximate the solution equations of many 
systems of partial differential equations using data. For future 
work, we plan to apply our approach to control a pandemic and other 
related systems using SCIPPlan based on learned 
models~\cite{Wu2017,Say2018a, Say2019a,Say2020a,Say2020b,Say2020c,Wu2020,Say2021}.

\section{Conclusion}

In this paper, we have formalized the pandemic planning problem 
based on the solution equations of the SIR model and solved it 
using a metric hybrid planner. We have introduced 
valid inequalities to improve the runtime effectiveness of the 
planner. Finally, we have presented both theoretical and 
experimental results on the performance of our approach to 
pandemic planning. Overall, we have demonstrated the potential 
of using metric hybrid planning to help control pandemics and 
other related systems.

\section{Disclaimer}

This work studies the computational effectiveness of using the 
solution equations of an SIR model, and does not constitute as 
health advice and does not take the place of consulting with 
the experts of the relevant fields.

\bibliography{mybibfile.bib}

%%%%%%%%%%%%%%%%%%%%%%%%%%%%%%%%%%%%%%%%%%%%%%%%%%%%%%%%%%%%%%%%%%%%%%%%

\end{document}